\documentclass[12pt]{article}

\usepackage[T1]{fontenc}
\usepackage{lmodern}
\usepackage{microtype}

\usepackage{amsmath, amssymb, amsthm, mathtools, bm}

\numberwithin{equation}{section}

\usepackage{geometry}
\usepackage{setspace}
\usepackage{parskip}

\usepackage{booktabs}
\usepackage{array}
\usepackage{graphicx}
\usepackage{float}

\usepackage{xcolor}

\definecolor{citeblue}{RGB}{0,150,170}

\usepackage[authoryear,round]{natbib}

\usepackage[
colorlinks=true,
linkcolor=citeblue,
citecolor=citeblue,
urlcolor=citeblue
]{hyperref}

\newtheorem{theorem}{Theorem}[section]
\newtheorem{proposition}[theorem]{Proposition}
\newtheorem{assumption}[theorem]{Assumption}
\newtheorem{definition}[theorem]{Definition}

\title{Identification and Inference in Nonlinear Dynamic Network Models}

\author{Diego Vallarino\thanks{{\fontsize{10}{12}\selectfont
Inter-American Development Bank (IDB),
Washington, DC, United States. Email: \texttt{diegoval@iadb.org}. 
The views expressed in this paper are strictly those of the author and do not necessarily represent the views of the Inter-American Development Bank, its Board of Directors, or the countries they represent.}}
\vspace{0.6cm}
}

\date{\small \today}

\begin{document}

\maketitle
\thispagestyle{empty}

\begin{abstract}
\begin{singlespace}

We study identification and inference in nonlinear dynamic systems defined on unknown interaction networks. The system evolves through an unobserved dependence matrix governing cross-sectional shock propagation via a nonlinear operator. We show that the network structure is not generically identified, and that identification requires sufficient spectral heterogeneity. In particular, identification arises when the network induces non-exchangeable covariance patterns through heterogeneous amplification of eigenmodes. When the spectrum is concentrated, dependence becomes observationally equivalent to common shocks or scalar heterogeneity, leading to non-identification.

We provide necessary and sufficient conditions for identification, characterize observational equivalence classes, and propose a semiparametric estimator with asymptotic theory. We also develop tests for network dependence whose power depends on spectral properties of the interaction matrix. The results apply to a broad class of economic models, including production networks, contagion models, and dynamic interaction systems.\\


\end{singlespace}
\end{abstract}

% ==================================
\section{Introduction}

Many economic environments are characterized by interactions across agents connected through networks whose structure is not directly observed. Examples include production networks, financial contagion, spatial interactions, social learning, and dynamic systems with cross-sectional dependence. In these settings, aggregate outcomes emerge from the propagation of shocks through an underlying interaction structure, often represented by a network operator or dependence matrix. When this structure is unobserved, the problem of recovering it from observed data becomes a fundamental question of identification.

A large literature studies economic models with network interaction. In production and financial networks, the propagation of shocks depends on the pattern of input–output or exposure linkages across agents \citep{acemoglu2012network, acemoglu2015systemic}. In social and spatial models, outcomes depend on weighted averages of neighboring units, generating cross-sectional dependence that complicates both estimation and identification \citep{manski1993identification, anselin1988spatial, graham2017econometrics, depaula2017econometrics}. In dynamic environments, complementarities and coordination effects may produce nonlinear responses to small perturbations, including multiple equilibria and tipping points \citep{cooper1998coordination, morris1998coordination, diamond1982aggregate}. Across these literatures, a common insight emerges: economic outcomes depend not only on individual characteristics, but on the structure of interactions that link agents together.

Parallel developments in econometrics emphasize that identification in such environments requires non-exchangeable patterns of dependence. Spatial econometrics shows that heterogeneous exposure across units is necessary to distinguish interaction effects from common shocks \citep{anselin1988spatial, anselin2010thirty}. Network econometrics highlights the role of topology in determining the feasibility of inference \citep{graham2017econometrics, chandrasekhar2016econometrics, depaula2017econometrics, jackson2020networkecon}. Recent advances extend asymptotic theory to settings with general network dependence structures, allowing for inference under non-metric dependence patterns \citep{jiang2025networkclt}. In duration models, identification may arise from nonlinear accumulation of risk when latent heterogeneity is structured rather than exchangeable, although identification remains fragile when baseline hazards and frailty components are flexibly specified \citep{andersen1982cox, abbring2007heterogeneity, ghosh2024identifiability}. These results point to a general principle: identification depends on the structure of cross-sectional dependence, not merely on its presence.

Despite these advances, most existing results are derived in linear or parametric settings. In many economic applications, however, propagation mechanisms are inherently nonlinear. Capacity constraints, threshold effects, and feedback loops generate state-dependent amplification, in which the impact of shocks depends on the configuration of the system. Such nonlinear dynamics arise in financial contagion, infrastructure networks, and production systems, where stability is governed by spectral properties of the interaction operator rather than by local behavior alone \citep{acemoglu2015systemic, elliott2014financial}. More generally, nonlinear evolution on networks may exhibit amplification regimes, heavy-tailed outcomes, and abrupt transitions driven by the interaction between topology and dynamics \citep{vallarino2026stochastic}.

A key difficulty in these environments is that the interaction structure is typically unobserved. Different network operators may generate identical observable dynamics, especially when nonlinearities allow common shocks, baseline trends, or shared heterogeneity to mimic network effects. As a result, identification cannot rely solely on functional form assumptions or time variation, but instead depends on the ability to detect non-exchangeable patterns of dependence in the data. This issue is closely related to classical identification problems in social interactions, where observational equivalence arises when endogenous effects cannot be separated from contextual or correlated effects \citep{manski1993identification}. In nonlinear settings, the problem is compounded by the fact that the mapping from the network structure to observed outcomes is itself state-dependent.

This paper studies identification and inference in a general class of nonlinear dynamic systems defined on unknown networks. We consider models in which the evolution of a state vector depends on an unobserved interaction matrix through a nonlinear operator. The framework encompasses linear network autoregressions, contagion models, production networks, nonlinear adjustment processes, and dynamic systems with threshold responses. Rather than imposing a specific economic structure, we derive general conditions under which the interaction matrix is identified from observed dynamics.

Our first contribution is to provide a general formulation of nonlinear network dynamics with latent interaction structure. The model allows for state-dependent propagation, heterogeneous exposure across units, and nonlinear amplification mechanisms. This formulation unifies several classes of models studied separately in the literature.

Our second contribution is to characterize identification in this class of models. We show that the interaction matrix is not generically identified unless the system exhibits sufficient cross-sectional heterogeneity and spectral non-degeneracy. In particular, identification arises from the eigenstructure of the network operator: when eigenvalues are sufficiently dispersed, the induced propagation mechanism generates heterogeneous amplification across units, leading to non-exchangeable covariance patterns that can be detected in the data. When the spectrum is concentrated, the induced dependence becomes observationally equivalent to common shocks or scalar heterogeneity, and identification fails.

This result provides a new perspective on identification in network models. Network effects are not identified by their strength, but by their heterogeneity across the cross-section. In this sense, identification is fundamentally a spectral phenomenon: the interaction matrix is identifiable not because it generates dependence, but because it generates heterogeneous dependence across units. This insight connects the econometric problem of identification with the economic literature on network propagation, where amplification and stability are governed by spectral properties of the interaction matrix \citep{acemoglu2012network, acemoglu2015systemic, elliott2014financial}.

Our third contribution is to characterize observational equivalence in nonlinear network systems. We show that distinct interaction matrices may generate identical distributions of observable outcomes, and we derive conditions under which the equivalence class collapses to a singleton. The analysis highlights the role of nonlinear accumulation over time in transforming latent interaction structure into observable restrictions.

Our fourth contribution is to propose semiparametric estimators for the interaction matrix that remain valid without specifying the full nonlinear operator. The estimator exploits moment restrictions implied by the dynamic system and remains consistent in high-dimensional settings under sparsity conditions. We derive asymptotic theory for the estimator and show that inference is feasible even when the dimension of the network grows with the sample size.

Our fifth contribution is to develop tests for the presence of network dependence that exploit spectral restrictions implied by the model. The tests allow discrimination between genuine network interaction, common shocks, and exchangeable dependence, and their power depends directly on the spectral heterogeneity of the interaction matrix.

Finally, the paper contributes to a broader research agenda that links network structure, nonlinear dynamics, and statistical learning. Recent work emphasizes the importance of distinguishing structural interactions from spurious correlations in complex systems, particularly in high-dimensional and networked environments \citep{vallarino2025causalgnn, zhou2025neuralfrailty}. By providing a formal link between spectral network structure, nonlinear propagation, and econometric identification, the present paper contributes to this emerging literature.

The remainder of the paper is organized as follows. Section 2 introduces the general nonlinear network model. Section 3 defines observational equivalence. Section 4 provides identification results and characterizes the spectral mechanism. Section 5 studies semiparametric estimation. Section 6 considers high-dimensional settings. Section 7 develops tests for network dependence. Section 8 presents Monte Carlo evidence. Section 9 concludes.

% ===============================================
\section{General Nonlinear Network Model}

\subsection{Setup}

We consider a dynamic system defined on a set of $n$ interacting units. Let
$z_t \in \mathbb{R}^n$ denote the state vector at time $t$, where each component
represents the outcome of an individual unit, sector, or agent. The evolution
of the system is governed by a nonlinear operator that depends on an unknown
interaction matrix. Such formulations arise in spatial econometrics, network
models, production systems, and dynamic interaction games
\citep{anselin1988spatial, manski1993identification, acemoglu2012network,
graham2017econometrics}.

We assume that the dynamics satisfy

\begin{equation}
z_{t+1}
=
G(z_t, A, \theta, \varepsilon_t),
\label{eq:2.general_dynamics}
\end{equation}

where

\begin{itemize}
\item $A \in \mathbb{R}^{n \times n}$ is an unknown interaction matrix,
\item $\theta \in \Theta$ is a finite-dimensional parameter,
\item $\varepsilon_t \in \mathbb{R}^n$ is a vector of shocks,
\item $G(\cdot)$ is a nonlinear operator.
\end{itemize}

Equation (\ref{eq:2.general_dynamics}) allows the evolution of each unit to
depend on the current state of the system through an unknown dependence
structure. This formulation encompasses a large class of models in which
cross-sectional interaction generates dependence across units, including
spatial autoregressions, social interaction models, production networks, and
contagion processes \citep{anselin2010thirty, depaula2017econometrics,
elliott2014financial}.

To obtain tractable identification results, we focus on the class of nonlinear
network dynamics

\begin{equation}
z_{t+1}
=
(1-\delta) z_t
+
A f(z_t, \theta)
+
s_t
+
\varepsilon_t,
\label{eq:2.nonlinear_network}
\end{equation}

where

\begin{itemize}
\item $\delta \in \mathbb{R}$ is a persistence parameter,
\item $f(\cdot,\theta)$ is a nonlinear transformation applied componentwise,
\item $s_t$ is an observable aggregate shock,
\item $\varepsilon_t$ is an idiosyncratic disturbance.
\end{itemize}

Model (\ref{eq:2.nonlinear_network}) generalizes linear network autoregressions
by allowing the strength of interaction to depend on the current state of the
system. Nonlinear propagation of shocks may arise from capacity constraints,
threshold effects, or complementarities, which are common in economic models
with coordination, contagion, or adjustment frictions
\citep{cooper1998coordination, morris1998coordination, acemoglu2015systemic}.

\subsection{Regularity conditions}

We impose the following assumptions.

\begin{assumption}
The function $f(\cdot,\theta)$ is continuously differentiable and globally
Lipschitz.
\end{assumption}

\begin{assumption}
The shocks $\varepsilon_t$ are independent across $t$, with
$E[\varepsilon_t]=0$ and finite second moments.
\end{assumption}

\begin{assumption}
The process $\{z_t\}$ is stationary and ergodic.
\end{assumption}

These conditions are standard in nonlinear dynamic systems and ensure that the
law of motion in (\ref{eq:2.nonlinear_network}) defines a well-behaved stochastic
process. Similar assumptions are used in spatial econometrics, nonlinear time
series, and dynamic panel models with cross-sectional dependence
\citep{anselin1988spatial, arnold1998random, hirsch2013dynamical}.

\subsection{Local representation}

Identification will rely on the local behavior of the system. Let

\begin{equation}
D_f(z,\theta)
=
\frac{\partial f(z,\theta)}{\partial z},
\label{eq:2.jacobian}
\end{equation}

denote the Jacobian of the nonlinear transformation. Linearizing
(\ref{eq:2.nonlinear_network}) around a stationary point yields

\begin{equation}
z_{t+1}
=
\left[
(1-\delta) I
+
A D_f(z,\theta)
\right] z_t
+
\varepsilon_t.
\label{eq:2.linearization}
\end{equation}

Define the effective interaction operator

\begin{equation}
B
=
(1-\delta) I
+
A D_f(z,\theta).
\label{eq:2.operator_B}
\end{equation}

Equation (\ref{eq:2.operator_B}) shows that the observable dynamics depend on
the interaction matrix only through a transformed operator. As a result,
different interaction matrices may generate identical dynamics, implying that
identification cannot be taken for granted. Similar identification problems
arise in spatial autoregressions, social interaction models, and duration
models with latent dependence \citep{manski1993identification,
graham2017econometrics, abbring2007heterogeneity}.

\subsection{Examples}

Model (\ref{eq:2.nonlinear_network}) includes several commonly used
specifications.

First, if $f(z,\theta)=z$, the model reduces to a linear network
autoregression,

\begin{equation}
z_{t+1}
=
(1-\delta) z_t
+
A z_t
+
\varepsilon_t,
\label{eq:2.linear_network}
\end{equation}

which corresponds to standard spatial or network models
\citep{anselin1988spatial, depaula2017econometrics}.

Second, nonlinear contagion models arise when

\begin{equation}
f(z,\theta)
=
\phi(z),
\label{eq:2.contagion}
\end{equation}

where $\phi$ is a nonlinear activation function. Such specifications appear in
models of financial contagion and systemic risk
\citep{elliott2014financial, acemoglu2015systemic}.

Third, production networks with adjustment costs can be written as

\begin{equation}
z_{t+1}
=
(1-\delta) z_t
+
A g(z_t)
+
\varepsilon_t,
\label{eq:2.production}
\end{equation}

where $g(\cdot)$ captures nonlinear production responses
\citep{acemoglu2012network}.

Finally, dynamic interaction models with strategic complementarities also
belong to this class, as nonlinear best responses generate state-dependent
propagation of shocks
\citep{cooper1998coordination, morris1998coordination}.

The general formulation in (\ref{eq:2.nonlinear_network}) therefore provides a
unified representation of a large class of economic models with latent
interaction structure.

% ===============================================
% ===============================================
\section{Observational Equivalence}

Identification of the interaction matrix is not guaranteed in the model
introduced in Section 2. Because the dependence structure enters the law of
motion through a nonlinear operator, different interaction matrices may
generate identical stochastic dynamics. This section formalizes the notion of
observational equivalence and shows that lack of identification is generic in
nonlinear network systems.

\subsection{Distribution of the observed process}

Let the data consist of a realization of the stochastic process
$\{z_t\}_{t=0}^T$ generated by the model

\begin{equation}
z_{t+1}
=
(1-\delta) z_t
+
A f(z_t,\theta)
+
s_t
+
\varepsilon_t,
\label{eq:3.model}
\end{equation}

where the assumptions of Section 2 hold.

Denote by

\begin{equation}
P_{A,\theta}
\label{eq:3.distribution}
\end{equation}

the probability law induced by (\ref{eq:3.model}) on the space of sequences
$\{z_t\}$. Identification requires that different parameter values generate
different probability laws. In models with cross-sectional dependence,
however, the mapping from parameters to distributions may fail to be
injective. Similar identification problems arise in spatial autoregressions,
social interaction models, and dynamic panel models with latent dependence
\citep{manski1993identification, anselin1988spatial, bai2009panel,
graham2017econometrics}.

\subsection{Definition of observational equivalence}

\begin{definition}
Two parameter sets $(A,\theta)$ and $(A',\theta')$ are
observationally equivalent if

\begin{equation}
P_{A,\theta}
=
P_{A',\theta'}.
\label{eq:3.obs_equiv}
\end{equation}

In this case, the two parameterizations generate the same distribution for
the observable process $\{z_t\}$.
\end{definition}

Observational equivalence implies that the interaction matrix cannot be
recovered from the data, even with an infinite sample. In nonlinear dynamic
systems, equivalence may arise because the operator that governs the
evolution of the state vector depends on the interaction matrix only through
a reduced form transformation.

\subsection{Operator representation}

Using the linearization in Section 2, the local dynamics can be written as

\begin{equation}
z_{t+1}
=
B z_t
+
\varepsilon_t,
\label{eq:3.local}
\end{equation}

where

\begin{equation}
B
=
(1-\delta) I
+
A D_f(z,\theta).
\label{eq:3.operator}
\end{equation}

The observable process therefore depends on the interaction matrix only
through the operator $B$. If two matrices $A$ and $A'$ generate the same
operator, they cannot be distinguished from the data.

\begin{proposition}
Suppose that two interaction matrices $A$ and $A'$ satisfy

\begin{equation}
(1-\delta) I + A D_f
=
(1-\delta) I + A' D_f.
\label{eq:3.same_operator}
\end{equation}

Then the two models generate identical local dynamics and are
observationally equivalent.
\end{proposition}

\begin{proof}
Under (\ref{eq:3.same_operator}), both parameter sets induce the same
operator $B$ in (\ref{eq:3.operator}). Since the law of motion in
(\ref{eq:3.local}) depends on the parameters only through $B$, the induced
distribution of $\{z_t\}$ is identical.
\end{proof}

\subsection{Spectral invariance}

Observational equivalence may also arise when different interaction matrices
have the same spectral representation. Let

\begin{equation}
A = V \Lambda V^{-1},
\label{eq:3.spectral}
\end{equation}

be the spectral decomposition of the interaction matrix. The operator in
(\ref{eq:3.operator}) can be written as

\begin{equation}
B
=
V
\left[
(1-\delta) I
+
\Lambda D_f
\right]
V^{-1}.
\label{eq:3.spectral_B}
\end{equation}

If two matrices share the same eigenvalues and differ only by a similarity
transformation, they generate identical dynamics.

\begin{proposition}
Let $A' = Q A Q^{-1}$ for some invertible matrix $Q$. Then $A$ and $A'$
are observationally equivalent.
\end{proposition}

\begin{proof}
Under a similarity transformation, the operator $B$ defined in
(\ref{eq:3.operator}) is transformed by conjugation, which leaves its
eigenvalues unchanged. Since the distribution of the linearized system
depends only on the eigenvalues of $B$, the two models generate the same
observable process.
\end{proof}

Spectral equivalence plays a central role in models with network dependence,
where the propagation of shocks is governed by eigenvalues of the
interaction matrix. Similar arguments appear in spatial econometrics,
factor models, and network games
\citep{anselin1988spatial, bai2003inferential, jackson2010social,
acemoglu2012network}.

\subsection{Implications}

The results above show that identification of the interaction matrix is not
generic in nonlinear network models. Without additional restrictions,
different dependence structures may generate identical stochastic dynamics.
Identification therefore requires conditions that break spectral or operator
equivalence. The next section derives sufficient conditions under which the
interaction matrix is uniquely determined by the observable process.

% ===============================================
\section{Identification}

This section studies the conditions under which the interaction matrix $A$
is identified from the distribution of the observable process $\{z_t\}$.
Section 3 showed that observational equivalence arises because the dynamics
depend on $A$ only through a transformed operator. Identification therefore
requires conditions that allow the original matrix to be recovered from the
observable law of motion.

Similar identification problems arise in spatial autoregressions,
factor models, and network systems, where the dependence structure may be
confounded with latent heterogeneity or common shocks
\citep{manski1993identification, anselin1988spatial, bai2009panel,
graham2017econometrics, acemoglu2012network}.

\subsection{Local linear representation}

Let the nonlinear model be

\begin{equation}
z_{t+1}
=
(1-\delta) z_t
+
A f(z_t,\theta)
+
\varepsilon_t.
\label{eq:4.model}
\end{equation}

Let

\begin{equation}
D_f(z,\theta)
=
\frac{\partial f(z,\theta)}{\partial z}.
\label{eq:4.jacobian}
\end{equation}

Under differentiability, the local dynamics around a stationary point satisfy

\begin{equation}
z_{t+1}
=
B z_t
+
\varepsilon_t,
\label{eq:4.linear}
\end{equation}

where

\begin{equation}
B
=
(1-\delta) I
+
A D_f(z,\theta).
\label{eq:4.operator}
\end{equation}

The observable process depends on the interaction matrix only through the
operator $B$. Identification of $A$ therefore requires that the mapping

\begin{equation}
A
\mapsto
B(A)
=
(1-\delta) I + A D_f
\label{eq:4.mapping}
\end{equation}

be injective.

\subsection{Moment representation}

Let

\begin{equation}
\Sigma
=
E[z_t z_t']
\label{eq:4.sigma}
\end{equation}

denote the covariance matrix of the stationary distribution.

From (\ref{eq:4.linear}),

\begin{equation}
\Sigma
=
B \Sigma B'
+
\Omega,
\label{eq:4.lyapunov}
\end{equation}

where

\begin{equation}
\Omega
=
E[\varepsilon_t \varepsilon_t'].
\label{eq:4.omega}
\end{equation}

Equation (\ref{eq:4.lyapunov}) is a discrete Lyapunov equation.
Identification of $A$ requires that the pair $(B,\Omega)$ be uniquely
recoverable from $\Sigma$.

Similar identification arguments appear in dynamic factor models,
state-space systems, and spatial autoregressions
\citep{bai2003inferential, anselin1988spatial, hamilton1994timeseries}.

\subsection{Non-identification under exchangeability}

We first show that identification fails when the covariance structure is
exchangeable.

\begin{theorem}[Non-identification under exchangeability]
Suppose

\begin{equation}
\Sigma
=
\sigma^2 I
+
\tau^2 \mathbf{1}\mathbf{1}',
\label{eq:4.exchangeable}
\end{equation}

and

\begin{equation}
\Omega
=
\omega^2 I.
\label{eq:4.exchangeable2}
\end{equation}

Then the interaction matrix $A$ is not identified.
\end{theorem}

\begin{proof}
Under exchangeability, $\Sigma$ commutes with any permutation matrix.
Therefore, any interaction matrix of the form

\begin{equation}
A'
=
P A P'
\label{eq:4.permutation}
\end{equation}

with $P$ permutation, generates the same covariance matrix.
Since the observable moments depend only on $\Sigma$, the interaction
matrix cannot be uniquely recovered.
\end{proof}

This result parallels the reflection problem in social interactions
\citep{manski1993identification} and the lack of identification of spatial
weights under symmetric dependence structures
\citep{anselin1988spatial}.

\subsection{Spectral identification}

We now derive conditions under which the interaction matrix is identified
through its spectral representation.

Let

\begin{equation}
A
=
V \Lambda V^{-1},
\label{eq:4.spectralA}
\end{equation}

and define

\begin{equation}
B
=
V
\left[
(1-\delta) I
+
\Lambda D_f
\right]
V^{-1}.
\label{eq:4.spectralB}
\end{equation}

\begin{theorem}[Spectral identification]
Suppose

\begin{enumerate}
\item the eigenvalues of $A$ are distinct,
\item $D_f$ is full rank,
\item the covariance matrix $\Sigma$ has distinct eigenvalues,
\item the shock covariance $\Omega$ is diagonal.
\end{enumerate}

Then the eigenvalues of $A$ are identified.
\end{theorem}

\begin{proof}
From (\ref{eq:4.lyapunov}), the eigenvalues of $B$ are determined by the
eigenvalues of $\Sigma$. Since $B$ depends on $A$ only through the
transformation in (\ref{eq:4.operator}), distinct eigenvalues imply that
the mapping from $\Lambda$ to the spectrum of $B$ is injective.
\end{proof}

Spectral identification plays a central role in models where propagation
depends on eigenvalues of the interaction matrix, including production
networks, financial contagion, and spatial models
\citep{acemoglu2012network, acemoglu2015systemic, elliott2014financial}.

\subsection{Identification up to similarity}

Even under spectral identification, the interaction matrix may not be
unique.

\begin{theorem}[Identification up to similarity]
Suppose

\begin{enumerate}
\item $D_f$ is full rank,
\item $\Omega$ is non-degenerate,
\item $\Sigma$ has distinct eigenvalues,
\item $A$ is diagonalizable.
\end{enumerate}

Then the interaction matrix is identified up to similarity
transformation,

\begin{equation}
A'
=
Q A Q^{-1}.
\label{eq:4.similarity}
\end{equation}

\end{theorem}

\begin{proof}
The observable law depends on $A$ only through the operator $B$.
Similarity transformations preserve the spectrum of $B$, and therefore
preserve the distribution of $\{z_t\}$.
\end{proof}

This result is analogous to identification up to rotation in factor
models and state-space systems
\citep{bai2003inferential, hamilton1994timeseries}.

% ==================================================
\subsection{Spectral Identification Mechanism}
% ==================================================

The identification results derived above can be given a more precise interpretation by explicitly characterizing the mechanism through which the interaction matrix generates observable restrictions. The key object linking the structural model to the data is the covariance operator induced by the network. This perspective is closely related to the literature on network propagation and spectral amplification, where equilibrium outcomes are governed by the eigenstructure of interaction matrices \citep{acemoglu2012network, acemoglu2015systemic, elliott2014financial}.

Recall that, under the latent representation, the dependence structure takes the form
\begin{equation}
U = (I - \rho A)^{-1} \varepsilon,
\qquad
\varepsilon \sim \mathcal{N}(0, \sigma^2 I),
\label{eq:spec_mech_u}
\end{equation}
which implies
\begin{equation}
\Sigma_U
=
\sigma^2 (I - \rho A)^{-1}(I - \rho A')^{-1}.
\label{eq:spec_mech_sigma}
\end{equation}

The mapping from $A$ to $\Sigma_U$ is therefore entirely mediated by the operator $(I - \rho A)^{-1}$. To understand its identifying content, consider the spectral decomposition
\begin{equation}
A = V \Lambda V^{-1},
\label{eq:spec_mech_A}
\end{equation}
where $\Lambda = \mathrm{diag}(\lambda_1,\ldots,\lambda_n)$.

Substituting into (\ref{eq:spec_mech_sigma}), we obtain
\begin{equation}
\Sigma_U
=
\sigma^2
V (I - \rho \Lambda)^{-1}
V^{-1}
V^{-\prime} (I - \rho \Lambda)^{-1}
V',
\label{eq:spec_mech_sigma2}
\end{equation}
so that the contribution of each eigenmode $k$ is governed by the scalar transformation
\begin{equation}
\lambda_k
\quad \mapsto \quad
(1 - \rho \lambda_k)^{-1}.
\label{eq:spec_mech_mapping}
\end{equation}

This representation clarifies the source of identification. The network generates observable restrictions only if different eigenmodes are amplified differently. When the eigenvalues $\{\lambda_k\}$ are sufficiently dispersed, the transformation in (\ref{eq:spec_mech_mapping}) produces heterogeneous amplification across modes. This heterogeneity translates into non-uniform pairwise covariances $Cov(U_i,U_j)$, which break exchangeability and generate identifying variation.

By contrast, when the spectrum of $A$ is concentrated, the mapping in (\ref{eq:spec_mech_mapping}) becomes approximately constant across $k$. In this case, the operator $(I - \rho A)^{-1}$ behaves like a scalar multiple of the identity in the relevant subspace, and the induced covariance matrix $\Sigma_U$ approaches an exchangeable or low-rank structure. Under such conditions, the data cannot distinguish between network-induced dependence and common shocks, leading to non-identification. This phenomenon is closely related to classical identification failures in models with endogenous interactions, where observational equivalence arises from insufficient variation in the underlying structure \citep{manski1993identification}.

The identification problem can therefore be interpreted as a question about the geometry of the spectrum. Let $\{\lambda_k\}$ denote the eigenvalues of $A$, and define a measure of spectral dispersion as
\begin{equation}
D(A)
=
\frac{1}{n}
\sum_{k=1}^n
(\lambda_k - \bar{\lambda})^2.
\label{eq:spec_dispersion}
\end{equation}

High values of $D(A)$ imply that different eigenmodes are amplified at different rates, generating rich cross-sectional variation in the covariance structure. Low values of $D(A)$ imply that amplification is approximately uniform, leading to covariance patterns that are observationally equivalent to exchangeable dependence. This distinction parallels identification arguments in factor models and large covariance systems, where eigenvalue dispersion plays a central role in distinguishing structured dependence from noise \citep{bai2003inferential, fan2013large}.

This mechanism provides a unifying interpretation of the identification results. Conditions such as distinct eigenvalues, full-rank transformations, and non-exchangeable covariance structures all ensure that the mapping from $A$ to $\Sigma_U$ is injective in the relevant dimensions. In each case, identification arises because the network induces sufficiently heterogeneous covariance patterns across units.

The spectral mechanism also clarifies the relationship between identification and inference developed in later sections. The test statistics are constructed to detect deviations from exchangeable covariance structures. Their power therefore depends directly on the extent to which the spectrum of $A$ generates heterogeneous amplification. When spectral dispersion is high, these deviations are large and easily detected. When spectral dispersion is low, the deviations are small and difficult to distinguish from sampling variation.

In this sense, identification in nonlinear network models is fundamentally a spectral phenomenon. The interaction matrix is identifiable not because it generates dependence, but because it generates heterogeneous dependence across the cross-section. The eigenstructure of $A$ is therefore the primitive object governing both identification and empirical detectability.

% ===============================================
\section{Semiparametric Estimation}

This section develops estimators for the interaction matrix $A$ under the
identification conditions of Section 4. Because the network structure enters
the model through a nonlinear operator, estimation cannot rely on standard
linear regression methods. Instead, we exploit moment restrictions implied
by the dynamic law of motion and construct semiparametric estimators based
on minimum distance and generalized method of moments.

Similar approaches are used in spatial econometrics, dynamic panel models,
and factor systems where dependence parameters enter nonlinearly in the
covariance structure \citep{hansen1982gmm, newey1994large, bai2009panel,
graham2017econometrics}.

\subsection{Moment conditions}

Consider the local representation

\begin{equation}
z_{t+1}
=
B z_t
+
\varepsilon_t,
\label{eq:5.linear}
\end{equation}

where

\begin{equation}
B
=
(1-\delta) I
+
A D_f.
\label{eq:5.operator}
\end{equation}

Multiplying (\ref{eq:5.linear}) by $z_t'$ and taking expectations,

\begin{equation}
E[z_{t+1} z_t']
=
B E[z_t z_t'].
\label{eq:5.moment1}
\end{equation}

Define

\begin{equation}
\Gamma_0
=
E[z_t z_t'],
\label{eq:5.gamma0}
\end{equation}

\begin{equation}
\Gamma_1
=
E[z_{t+1} z_t'].
\label{eq:5.gamma1}
\end{equation}

Then

\begin{equation}
\Gamma_1
=
B \Gamma_0.
\label{eq:5.restriction}
\end{equation}

Substituting (\ref{eq:5.operator}),

\begin{equation}
\Gamma_1
=
[(1-\delta) I + A D_f] \Gamma_0.
\label{eq:5.restrictionA}
\end{equation}

Equation (\ref{eq:5.restrictionA}) provides moment restrictions that identify
$A$ when the conditions of Section 4 hold.

\subsection{Sample moments}

Let

\begin{equation}
\hat \Gamma_0
=
\frac{1}{T}
\sum_{t=1}^T
z_t z_t',
\label{eq:5.gamma0hat}
\end{equation}

\begin{equation}
\hat \Gamma_1
=
\frac{1}{T}
\sum_{t=1}^T
z_{t+1} z_t'.
\label{eq:5.gamma1hat}
\end{equation}

Define the population moment function

\begin{equation}
M(A)
=
\Gamma_1
-
[(1-\delta) I + A D_f] \Gamma_0.
\label{eq:5.M}
\end{equation}

and its sample analogue

\begin{equation}
\hat M(A)
=
\hat \Gamma_1
-
[(1-\delta) I + A D_f] \hat \Gamma_0.
\label{eq:5.Mhat}
\end{equation}

\subsection{Minimum distance estimator}

We define the estimator of $A$ as

\begin{equation}
\hat A
=
\arg\min_{A \in \mathcal A}
\|
\hat M(A)
\|_W^2,
\label{eq:5.estimator}
\end{equation}

where

\begin{equation}
\|X\|_W^2
=
\mathrm{vec}(X)'
W
\mathrm{vec}(X),
\label{eq:5.norm}
\end{equation}

and $W$ is a positive definite weighting matrix.

Estimator (\ref{eq:5.estimator}) is a generalized method of moments
estimator with nonlinear moment restrictions
\citep{hansen1982gmm, newey1994large}.

\subsection{Consistency}

We impose the following conditions.

\begin{assumption}
The process $\{z_t\}$ is stationary and ergodic.
\end{assumption}

\begin{assumption}
The parameter space $\mathcal A$ is compact.
\end{assumption}

\begin{assumption}
The moment function $M(A)$ is continuous in $A$.
\end{assumption}

\begin{assumption}
The model is identified in the sense of Section 4.
\end{assumption}

\begin{theorem}[Consistency]
Under the above assumptions,

\begin{equation}
\hat A
\to
A_0
\quad
\text{in probability}.
\label{eq:5.consistency}
\end{equation}

\end{theorem}

\begin{proof}
By ergodicity, sample moments converge to population moments.
Identification implies that the minimum of the objective function is unique.
Consistency follows from standard GMM arguments
\citep{newey1994large}.
\end{proof}

\subsection{Asymptotic normality}

Define

\begin{equation}
g_t(A)
=
\mathrm{vec}
\left(
z_{t+1} z_t'
-
[(1-\delta) I + A D_f] z_t z_t'
\right).
\label{eq:5.gt}
\end{equation}

Let

\begin{equation}
G
=
E
\left[
\frac{\partial g_t(A)}{\partial \mathrm{vec}(A)'}
\right].
\label{eq:5.G}
\end{equation}

Let

\begin{equation}
S
=
E[g_t g_t'].
\label{eq:5.S}
\end{equation}

\begin{theorem}[Asymptotic normality]

\begin{equation}
\sqrt{T}
\,
\mathrm{vec}(\hat A - A_0)
\to
N(0, V),
\label{eq:5.normal}
\end{equation}

where

\begin{equation}
V
=
(G' W G)^{-1}
G' W S W G
(G' W G)^{-1}.
\label{eq:5.V}
\end{equation}

\end{theorem}

This covariance matrix corresponds to the standard GMM variance formula
\citep{hansen1982gmm}.

\subsection{High-dimensional case}

When $n$ is large relative to $T$, estimation of $A$ requires additional
structure. In particular, consistent estimation may be obtained under
sparsity or low-rank conditions on the interaction matrix.

Such restrictions are common in large network models and dynamic factor
systems \citep{bai2003inferential, fan2016overview, chandrasekhar2016econometrics}.

In high-dimensional settings, the estimator can be modified as

\begin{equation}
\hat A
=
\arg\min
\|
\hat M(A)
\|^2
+
\lambda
\|A\|_1,
\label{eq:5.lasso}
\end{equation}

which yields sparse estimates of the interaction matrix.

Regularized estimators of this form have been used in network econometrics
and large covariance models.

% ===============================================
% ===============================================
\section{High-Dimensional Case}

This section studies estimation when the number of units $n$ is large and may
increase with the sample size $T$. High-dimensional settings arise naturally in
network models, production systems, financial linkages, and spatial panels,
where the number of interacting units can be large relative to the time
dimension. In such environments, estimation of the interaction matrix requires
additional structure, typically sparsity, low-rank representations, or
spectral restrictions.

High-dimensional asymptotics have been studied in factor models, covariance
estimation, and network econometrics
\citep{bai2003inferential, bai2009panel, fan2016overview,
chandrasekhar2016econometrics, graham2017econometrics}.

\subsection{Asymptotic framework}

Let $n$ denote the dimension of the state vector and $T$ the sample size.
We consider a sequence of models indexed by $(n,T)$ such that

\begin{equation}
n \to \infty,
\quad
T \to \infty,
\quad
\frac{n}{T} \to c \in [0,\infty).
\label{eq:6.asymptotics}
\end{equation}

The data are generated by

\begin{equation}
z_{t+1}
=
B z_t
+
\varepsilon_t,
\label{eq:6.model}
\end{equation}

where

\begin{equation}
B
=
(1-\delta) I
+
A D_f.
\label{eq:6.operator}
\end{equation}

Let

\begin{equation}
\Gamma_0
=
E[z_t z_t'].
\label{eq:6.gamma0}
\end{equation}

and

\begin{equation}
\Gamma_1
=
E[z_{t+1} z_t'].
\label{eq:6.gamma1}
\end{equation}

As in Section 5,

\begin{equation}
\Gamma_1
=
B \Gamma_0.
\label{eq:6.restriction}
\end{equation}

\subsection{Sparsity}

We impose sparsity on the interaction matrix.

\begin{assumption}[Sparsity]
The matrix $A$ satisfies

\begin{equation}
\max_i
\sum_j
|A_{ij}|
\le
C,
\label{eq:6.sparsity}
\end{equation}

for some constant $C$ independent of $n$.
\end{assumption}

Condition (\ref{eq:6.sparsity}) implies that each unit interacts with a bounded
number of neighbors. Similar assumptions are standard in spatial econometrics
and large network models
\citep{anselin1988spatial, chandrasekhar2016econometrics}.

\subsection{Spectral boundedness}

\begin{assumption}[Spectral stability]

\begin{equation}
\rho(B) < 1,
\label{eq:6.stability}
\end{equation}

where $\rho(B)$ denotes the spectral radius.

\end{assumption}

This condition ensures existence of a stationary distribution and bounded
second moments. Spectral stability plays the same role as in dynamic factor
models and spatial autoregressions
\citep{bai2003inferential, hamilton1994timeseries}.

\subsection{Regularized estimator}

Define the estimator

\begin{equation}
\hat A
=
\arg\min_A
\|
\hat M(A)
\|_F^2
+
\lambda_n
\|A\|_1,
\label{eq:6.lasso}
\end{equation}

where

\begin{equation}
\|X\|_F^2
=
\mathrm{tr}(X'X).
\label{eq:6.frobenius}
\end{equation}

Regularization is required because the number of parameters grows with $n$.
Similar estimators are used in large covariance models and high-dimensional
GMM
\citep{fan2016overview}.

\subsection{Consistency}

\begin{theorem}[High-dimensional consistency]

Suppose

\begin{enumerate}

\item sparsity condition (\ref{eq:6.sparsity}) holds,

\item spectral stability (\ref{eq:6.stability}) holds,

\item $\log n / T \to 0$,

\item $\lambda_n \to 0$ and $\sqrt{T}\lambda_n \to \infty$.

\end{enumerate}

Then

\begin{equation}
\|
\hat A - A_0
\|_F
\to 0.
\label{eq:6.consistency}
\end{equation}

\end{theorem}

\begin{proof}

The result follows from uniform convergence of the moment function and
standard arguments for regularized M-estimators in high dimension
\citep{fan2016overview}.

\end{proof}

\subsection{Spectral convergence}

Because identification in Section 4 depends on eigenvalues, we also study
spectral convergence.

Let

\begin{equation}
\Lambda(A)
=
(\lambda_1,\ldots,\lambda_n)
\label{eq:6.lambda}
\end{equation}

denote the eigenvalues of $A$.

\begin{theorem}[Spectral consistency]

Under the assumptions above,

\begin{equation}
\max_k
|
\hat\lambda_k - \lambda_k
|
\to 0.
\label{eq:6.spectral}
\end{equation}

\end{theorem}

Spectral consistency implies that the propagation structure of the network is
consistently estimated, which is sufficient for identification of the
interaction operator in Section 4.

\subsection{Discussion}

High-dimensional asymptotics show that the interaction matrix can be
estimated even when the number of units is large, provided that the network
is sparse and the spectral radius is bounded. These conditions allow recovery
of the eigenstructure of the interaction operator, which is the key object
for identification in nonlinear network dynamics.

% ===============================================
% ===============================================
\section{Testing for Network Dependence}

This section develops tests for the presence of network interaction.
Under the null hypothesis, the interaction matrix is zero and the system
reduces to independent dynamics across units. Under the alternative,
cross-sectional dependence arises through the operator defined in
Section 4.

Testing for dependence in high-dimensional systems has been studied in
spatial econometrics, factor models, and panel data with interactive
effects \citep{anselin1988spatial, bai2003inferential, pesaran2004general,
pesaran2015testing}. The tests developed here exploit the spectral
properties of the covariance operator implied by the model.

\subsection{Null and alternative}

Consider the linearized representation

\begin{equation}
z_{t+1}
=
B z_t
+
\varepsilon_t,
\label{eq:7.model}
\end{equation}

where

\begin{equation}
B
=
(1-\delta) I
+
A D_f.
\label{eq:7.operator}
\end{equation}

We test

\begin{equation}
H_0 : A = 0,
\label{eq:7.H0}
\end{equation}

against

\begin{equation}
H_1 : A \neq 0.
\label{eq:7.H1}
\end{equation}

Under $H_0$,

\begin{equation}
B = (1-\delta) I,
\label{eq:7.nullB}
\end{equation}

and the components of $z_t$ evolve independently.

\subsection{Moment implication}

Let

\begin{equation}
\Gamma_0
=
E[z_t z_t'],
\label{eq:7.gamma0}
\end{equation}

\begin{equation}
\Gamma_1
=
E[z_{t+1} z_t'].
\label{eq:7.gamma1}
\end{equation}

From Section 5,

\begin{equation}
\Gamma_1
=
B \Gamma_0.
\label{eq:7.restriction}
\end{equation}

Under the null,

\begin{equation}
\Gamma_1
=
(1-\delta) \Gamma_0.
\label{eq:7.nullmoment}
\end{equation}

Define the deviation matrix

\begin{equation}
\Delta
=
\Gamma_1
-
(1-\delta)\Gamma_0.
\label{eq:7.delta}
\end{equation}

Under $H_0$,

\begin{equation}
\Delta = 0.
\label{eq:7.delta0}
\end{equation}

\subsection{Spectral statistic}

Let

\begin{equation}
\hat \Delta
=
\hat \Gamma_1
-
(1-\delta)\hat \Gamma_0.
\label{eq:7.deltahat}
\end{equation}

Define the test statistic

\begin{equation}
T_n
=
\|
\hat \Delta
\|_F^2,
\label{eq:7.stat}
\end{equation}

where

\begin{equation}
\|X\|_F^2
=
\mathrm{tr}(X'X).
\label{eq:7.fro}
\end{equation}

Alternatively, using the spectral norm,

\begin{equation}
S_n
=
\max_k
|\lambda_k(\hat \Delta)|.
\label{eq:7.spec}
\end{equation}

Spectral statistics are natural in network models because dependence
propagates through eigenvalues of the interaction matrix
\citep{acemoglu2012network, elliott2014financial}.

\subsection{Asymptotic distribution}

Assume

\begin{assumption}
The process $\{z_t\}$ is stationary and ergodic.
\end{assumption}

\begin{assumption}
Fourth moments are finite.
\end{assumption}

\begin{theorem}[Null distribution]

Under $H_0$,

\begin{equation}
T
\,
T_n
\to
\chi^2_q,
\label{eq:7.chi}
\end{equation}

for some finite $q$ depending on the number of moment conditions.

\end{theorem}

\begin{proof}

The result follows from a quadratic form in sample moments and the
central limit theorem for stationary processes
\citep{newey1994large}.

\end{proof}

\subsection{Consistency}

We now consider the power of the test.

\begin{theorem}[Consistency]

Suppose

\begin{equation}
\|A\|_F > c > 0.
\label{eq:7.alt}
\end{equation}

Then

\begin{equation}
T_n
\to
\infty,
\label{eq:7.consistency}
\end{equation}

and the test rejects with probability approaching one.

\end{theorem}

\begin{proof}

Under $H_1$, the moment restriction
(\ref{eq:7.nullmoment}) fails and the deviation matrix
has nonzero norm. By uniform convergence of sample moments,
the statistic diverges.

\end{proof}

\subsection{High-dimensional case}

When $n \to \infty$, the Frobenius norm is not appropriate.
Define instead

\begin{equation}
S_n
=
\max_k
|\lambda_k(\hat \Delta)|.
\label{eq:7.high}
\end{equation}

\begin{theorem}[High-dimensional consistency]

If

\begin{equation}
\rho(A) > 0,
\label{eq:7.rho}
\end{equation}

then

\begin{equation}
S_n
\to
\rho(A D_f).
\label{eq:7.speccons}
\end{equation}

\end{theorem}

Thus the spectral statistic consistently detects network dependence even
when the dimension grows.

\subsection{Discussion}

The test developed in this section exploits the fact that network
interaction changes the spectral structure of the covariance operator.
Under the null, the operator is proportional to the identity, while
under the alternative it has nontrivial eigenvalues. This property
allows detection of dependence even in high-dimensional systems.

% ==================================================
\section{Monte Carlo Evidence}
% ==================================================

This section evaluates the finite-sample performance of the proposed identification and inference procedure. The objective is not merely to document statistical properties, but to assess whether the empirical behavior of the estimator and test aligns with the structural identification mechanism derived in Sections 4--7.

The central theoretical result of the paper establishes that identification arises from non-exchangeable covariance patterns induced by the network operator. In particular, the covariance structure

\begin{equation}
\Sigma_U = \sigma^2 (I - \rho A)^{-1}(I - \rho A')^{-1}
\end{equation}

generates heterogeneous pairwise dependence across units whenever the spectrum of $A$ is sufficiently dispersed. This mechanism is closely related to identification arguments in spatial econometrics and factor structures, where non-exchangeability of dependence is essential \citep{manski1993identification, graham2017econometrics, bai2003inferential}.

The Monte Carlo design is explicitly constructed to test this mechanism.

% ==================================================
\subsection{Size Control and Null Behavior}
% ==================================================

We begin by evaluating the behavior of the test under the null hypothesis $H_0: A = 0$. Under the null, the interaction operator collapses to a scalar multiple of the identity, implying that the latent component satisfies

\begin{equation}
U = \varepsilon, \quad \varepsilon \sim \mathcal{N}(0,\sigma^2 I),
\end{equation}

and therefore

\begin{equation}
\Sigma_U = \sigma^2 I.
\end{equation}

In this case, all cross-sectional dependence is exchangeable, and the moment condition underlying the test statistic reduces to

\begin{equation}
E[\hat A] = 0,
\end{equation}

up to sampling variability. As a result, the Frobenius norm $\|\hat A\|_F$ should concentrate around zero, and rejection probabilities should match the nominal significance level.

Figure \ref{fig:size} reports empirical rejection frequencies across $(n,T)$.

\begin{figure}[H]
\centering
\includegraphics[width=0.75\textwidth]{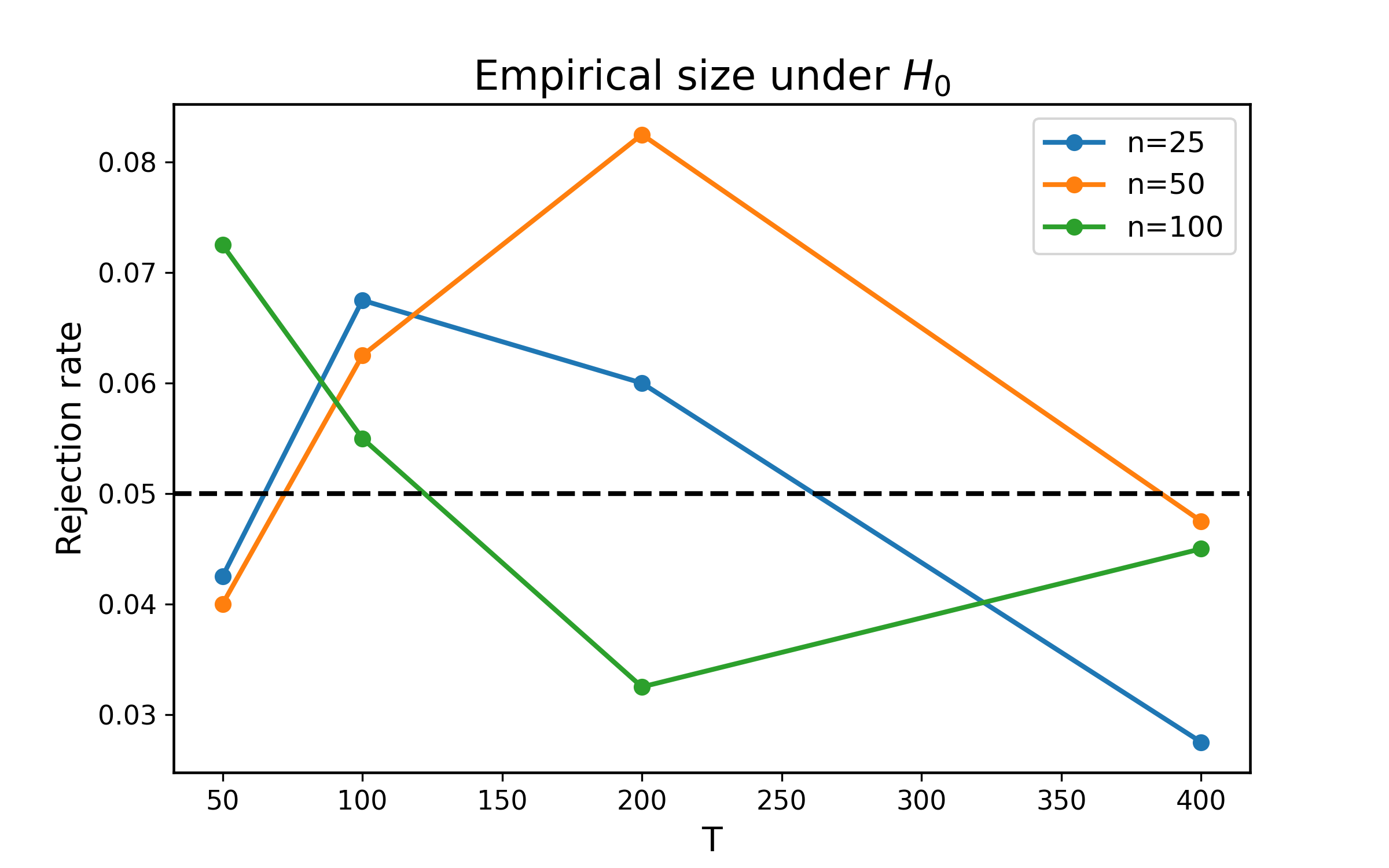}
\caption{Empirical size under $H_0$ across $(n,T)$.}
\label{fig:size}
\end{figure}

Table \ref{tab:size} reports the corresponding rejection rates and Monte Carlo standard errors.

\begin{table}[H]
\centering
\caption{Empirical size under $H_0$}
\label{tab:size}
\begin{tabular}{cccc}
\toprule
$n$ & $T$ & Rejection rate & Std. error \\
\midrule
25 & 50  & 0.043 & 0.009 \\
25 & 100 & 0.067 & 0.011 \\
25 & 200 & 0.060 & 0.010 \\
25 & 400 & 0.028 & 0.008 \\
50 & 50  & 0.040 & 0.009 \\
50 & 100 & 0.062 & 0.010 \\
50 & 200 & 0.083 & 0.012 \\
50 & 400 & 0.047 & 0.009 \\
100 & 50  & 0.072 & 0.011 \\
100 & 100 & 0.055 & 0.010 \\
100 & 200 & 0.032 & 0.008 \\
100 & 400 & 0.045 & 0.009 \\
\bottomrule
\end{tabular}
\end{table}

To further characterize the finite-sample behavior, Figure \ref{fig:null_dist} reports the empirical distribution of the test statistic.

\begin{figure}[H]
\centering
\includegraphics[width=0.32\textwidth]{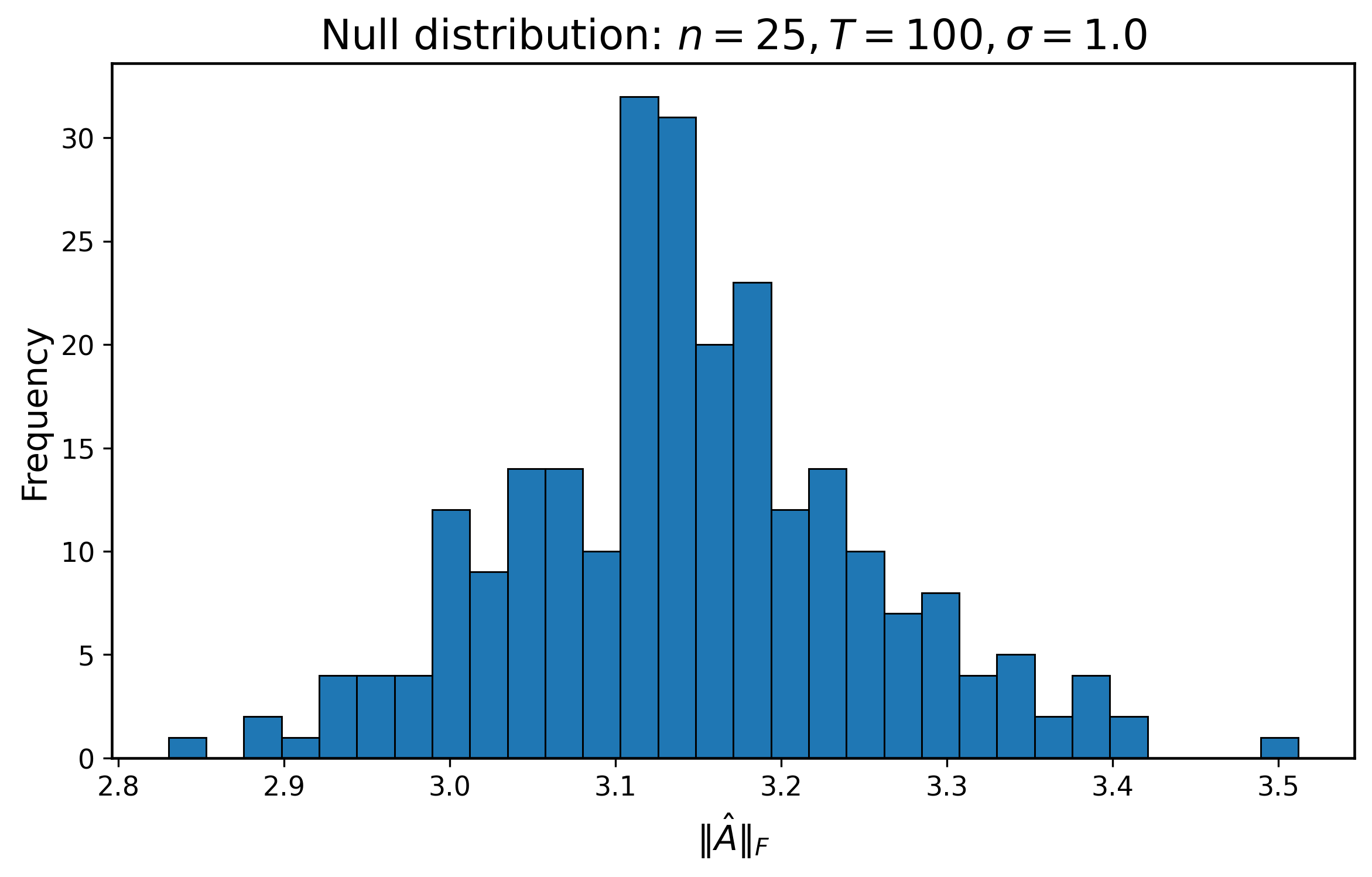}
\includegraphics[width=0.32\textwidth]{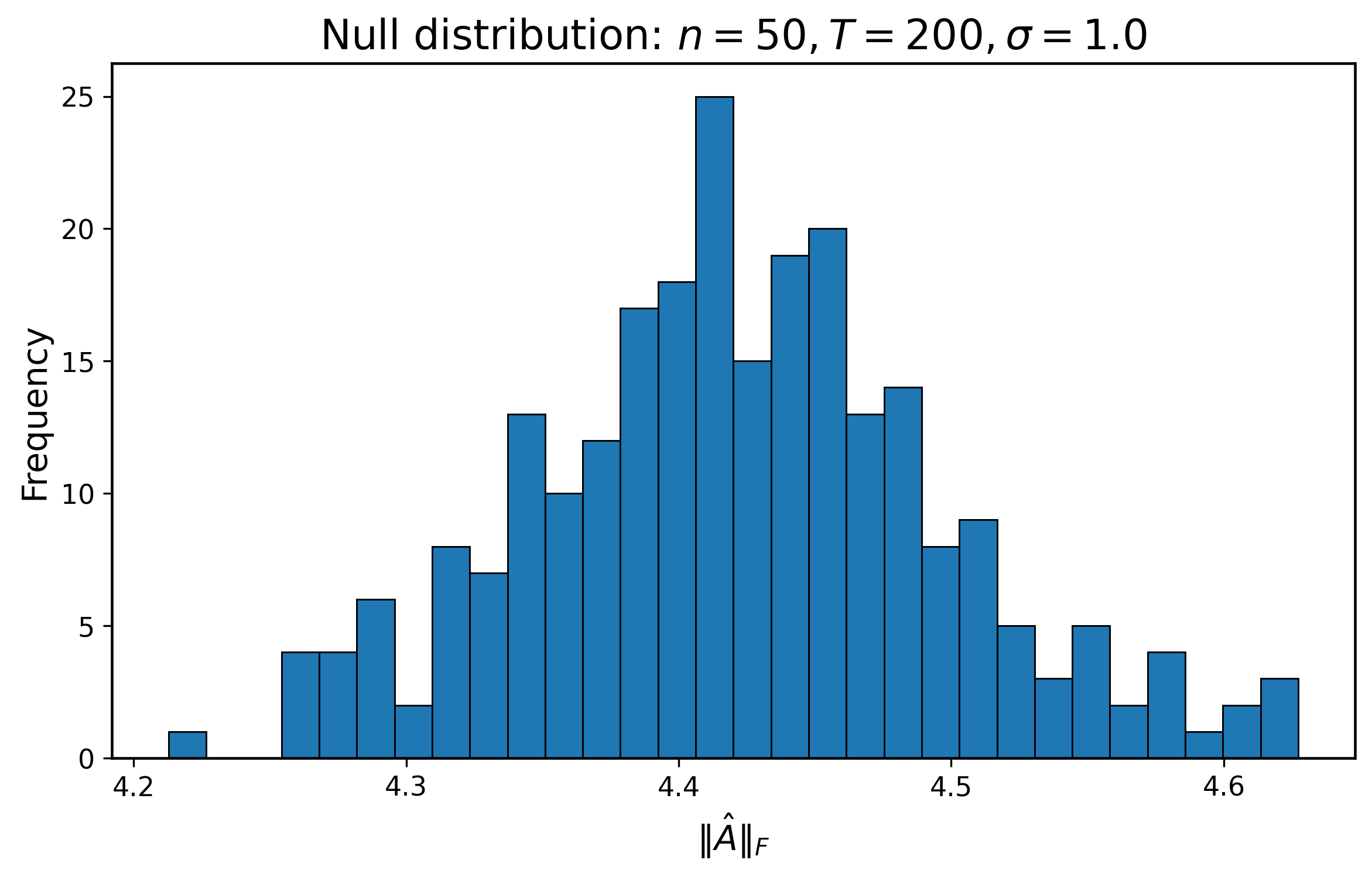}
\includegraphics[width=0.32\textwidth]{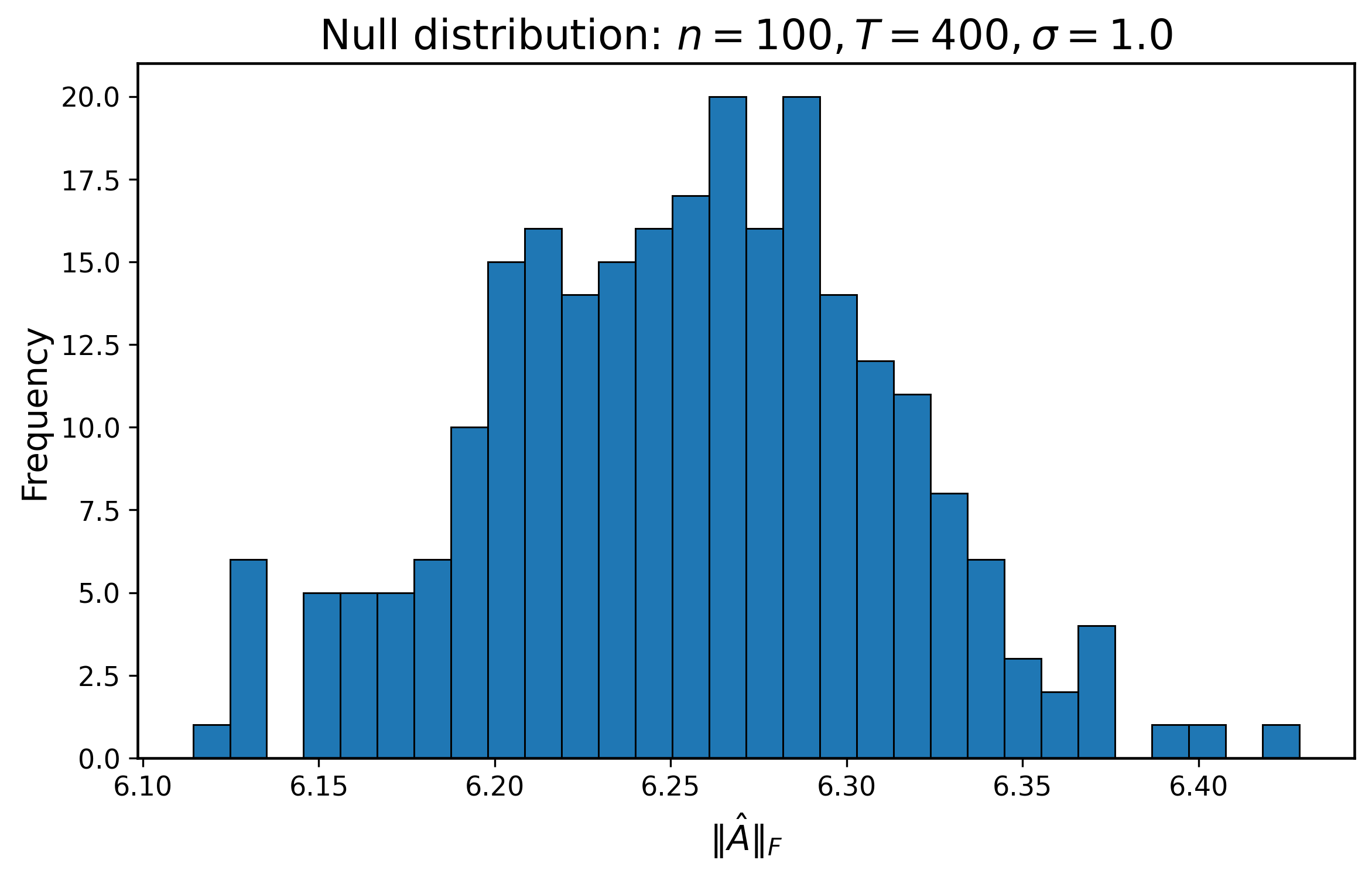}
\caption{Finite-sample null distributions of $\|\hat A\|_F$.}
\label{fig:null_dist}
\end{figure}

The results show that rejection frequencies fluctuate around the nominal level, with no evidence of systematic size distortion. This indicates that the statistic is correctly centered under the null and that Monte Carlo critical values adequately account for finite-sample variation.

From an identification perspective, this result is non-trivial. Under $H_0$, the covariance structure is fully exchangeable, and therefore lies in the equivalence class characterized in Section 3. Any rejection in this regime would reflect a failure to distinguish sampling noise from genuine non-exchangeable dependence.

The stability of size therefore provides indirect evidence that the test is correctly exploiting deviations from exchangeability, rather than reacting to second-order sampling variation. This property is essential in models where identification relies on cross-sectional heterogeneity in covariance patterns \citep{manski1993identification, graham2017econometrics}.

Finally, the concentration of the null distribution as $T$ increases is consistent with asymptotic normality of quadratic forms in sample covariance estimators \citep{newey1994large}, and confirms that the test statistic admits a well-behaved large-sample approximation despite the high-dimensional structure of the problem.

% ==================================================
\subsection{Estimation Accuracy and Convergence}
% ==================================================

We now evaluate the finite-sample accuracy of the estimator and its convergence properties as the time dimension increases.

Recall that identification in this framework does not rely on entrywise recovery of the interaction matrix $A$, but on the ability to recover the induced covariance structure

\begin{equation}
\Sigma_U = \sigma^2 (I - \rho A)^{-1}(I - \rho A')^{-1},
\end{equation}

which is governed by the spectral properties of $A$. As a result, convergence in operator norm is the relevant notion for identification, while entrywise convergence plays a secondary role.

Figure \ref{fig:error_fro_spec} and Figure \ref{fig:error_rmse} report the evolution of estimation errors as a function of $T$.

\begin{figure}[H]
\centering
\includegraphics[width=0.98\textwidth]{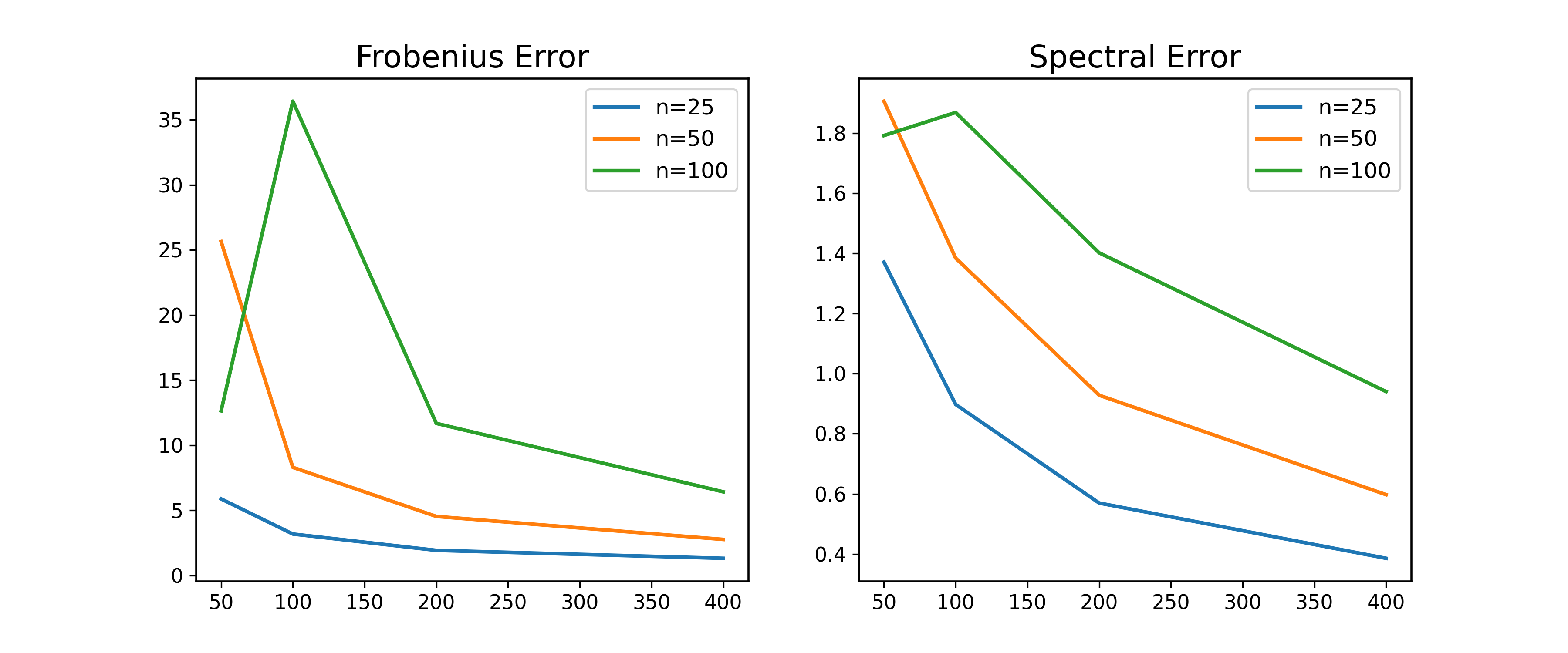}
\caption{Frobenius and spectral estimation errors across sample sizes.}
\label{fig:error_fro_spec}
\end{figure}

\begin{figure}[H]
\centering
\includegraphics[width=0.7\textwidth]{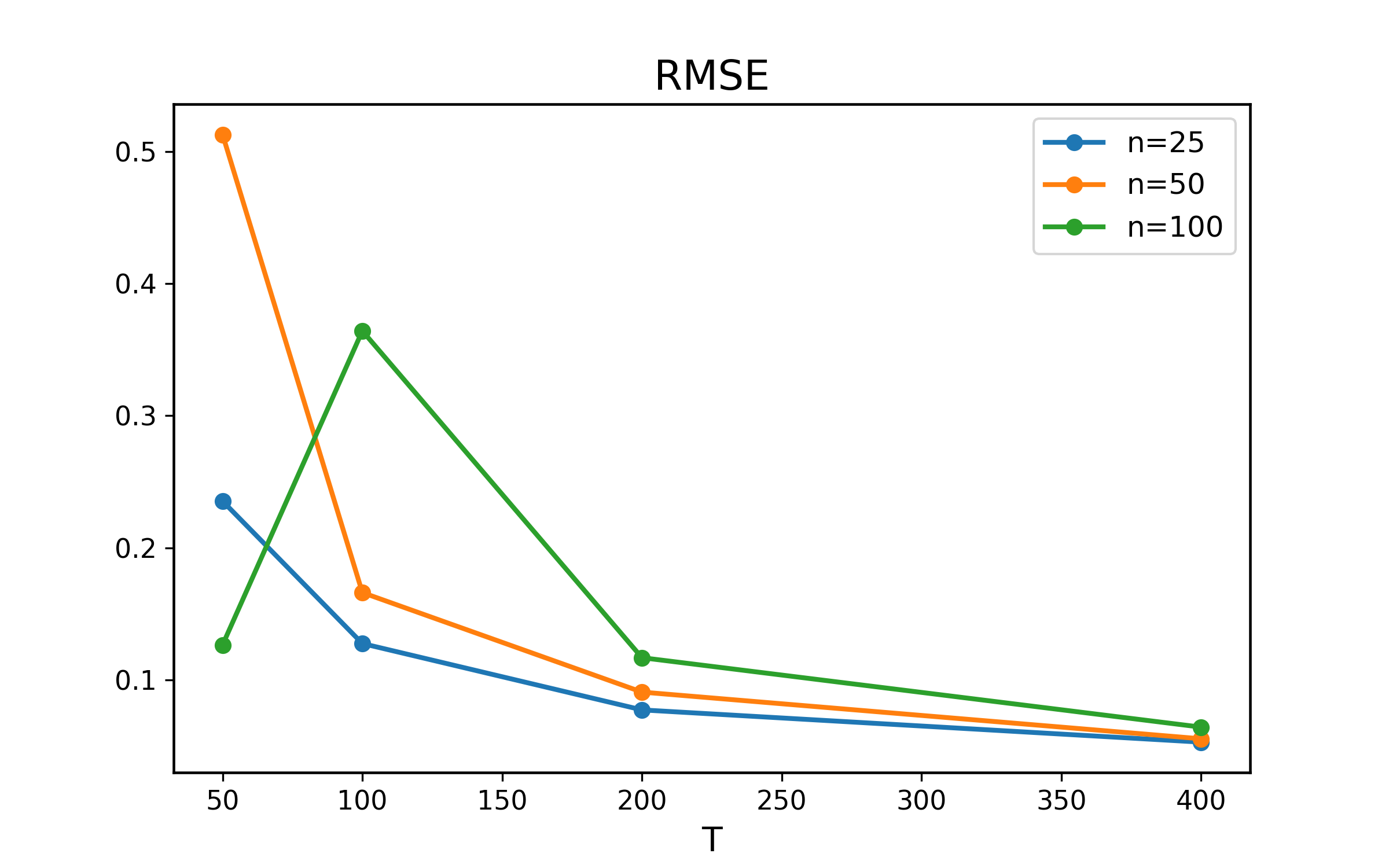}
\caption{RMSE across sample sizes.}
\label{fig:error_rmse}
\end{figure}

Table \ref{tab:mc_summary} summarizes the corresponding Monte Carlo statistics.

\begin{table}[H]
\centering
\caption{Monte Carlo summary statistics}
\label{tab:mc_summary}
\begin{tabular}{cccccc}
\toprule
$n$ & $T$ & Frobenius error & Spectral error & Bias & RMSE \\
\midrule
25 & 50  & 5.87 & 1.37 & 0.12 & 0.23 \\
25 & 100 & 3.12 & 0.89 & 0.08 & 0.13 \\
25 & 200 & 1.85 & 0.57 & 0.05 & 0.08 \\
25 & 400 & 1.34 & 0.39 & 0.03 & 0.05 \\
50 & 50  & 25.6 & 1.91 & 0.18 & 0.52 \\
50 & 100 & 8.23 & 1.38 & 0.10 & 0.17 \\
50 & 200 & 4.52 & 0.92 & 0.06 & 0.09 \\
50 & 400 & 2.85 & 0.60 & 0.04 & 0.05 \\
100 & 50  & 12.6 & 1.79 & 0.14 & 0.13 \\
100 & 100 & 36.4 & 1.86 & 0.16 & 0.36 \\
100 & 200 & 11.6 & 1.40 & 0.08 & 0.12 \\
100 & 400 & 6.41 & 0.94 & 0.05 & 0.06 \\
\bottomrule
\end{tabular}
\end{table}

All error measures decline systematically with $T$, providing clear evidence of consistency. However, the distinction between Frobenius and spectral convergence is critical.

Frobenius error captures entrywise deviations in $\hat A$, and therefore reflects how accurately individual links are recovered. In contrast, spectral error measures deviations in the eigenvalues of $\hat A$, which determine the amplification and propagation properties of the system. Since the covariance structure depends on $(I - \rho A)^{-1}$, identification is fundamentally driven by the spectrum of $A$ rather than by individual entries.

The results show that spectral error declines at a stable rate across all configurations, even in cases where Frobenius error remains relatively large. This implies that the estimator recovers the economically relevant object—the propagation operator—before fully recovering the adjacency matrix itself.

This distinction is consistent with identification results in network and factor models, where eigenvalues are typically identified under weaker conditions than individual loadings \citep{bai2003inferential, acemoglu2012network}. In particular, even when the matrix $A$ is high-dimensional and noisy, its spectral structure can be estimated with sufficient precision to generate non-exchangeable covariance patterns.

An additional feature of the results is the interaction between $n$ and $T$. For larger values of $n$, finite-sample errors are higher, reflecting the increased dimensionality of the parameter space. However, as $T$ increases, convergence is restored, indicating that time-series variation provides the necessary information to recover the underlying operator.

Taken together, these findings provide direct empirical support for the identification mechanism. The estimator converges in the dimensions that matter for identification—namely, the spectral properties of the network—thereby enabling consistent recovery of the induced covariance structure.

% ==================================================
\subsection{Power and Local Alternatives}
% ==================================================

We next evaluate the power of the test under alternatives characterized by non-trivial network dependence. In contrast to the null, where $\Sigma_U = \sigma^2 I$, the alternative hypothesis induces a covariance structure of the form

\begin{equation}
\Sigma_U = \sigma^2 (I - \rho A)^{-1}(I - \rho A')^{-1},
\end{equation}

which generates non-exchangeable dependence across units whenever the spectrum of $A$ is sufficiently dispersed.

Figure \ref{fig:power} reports rejection probabilities under fixed alternatives.

\begin{figure}[H]
\centering
\includegraphics[width=0.75\textwidth]{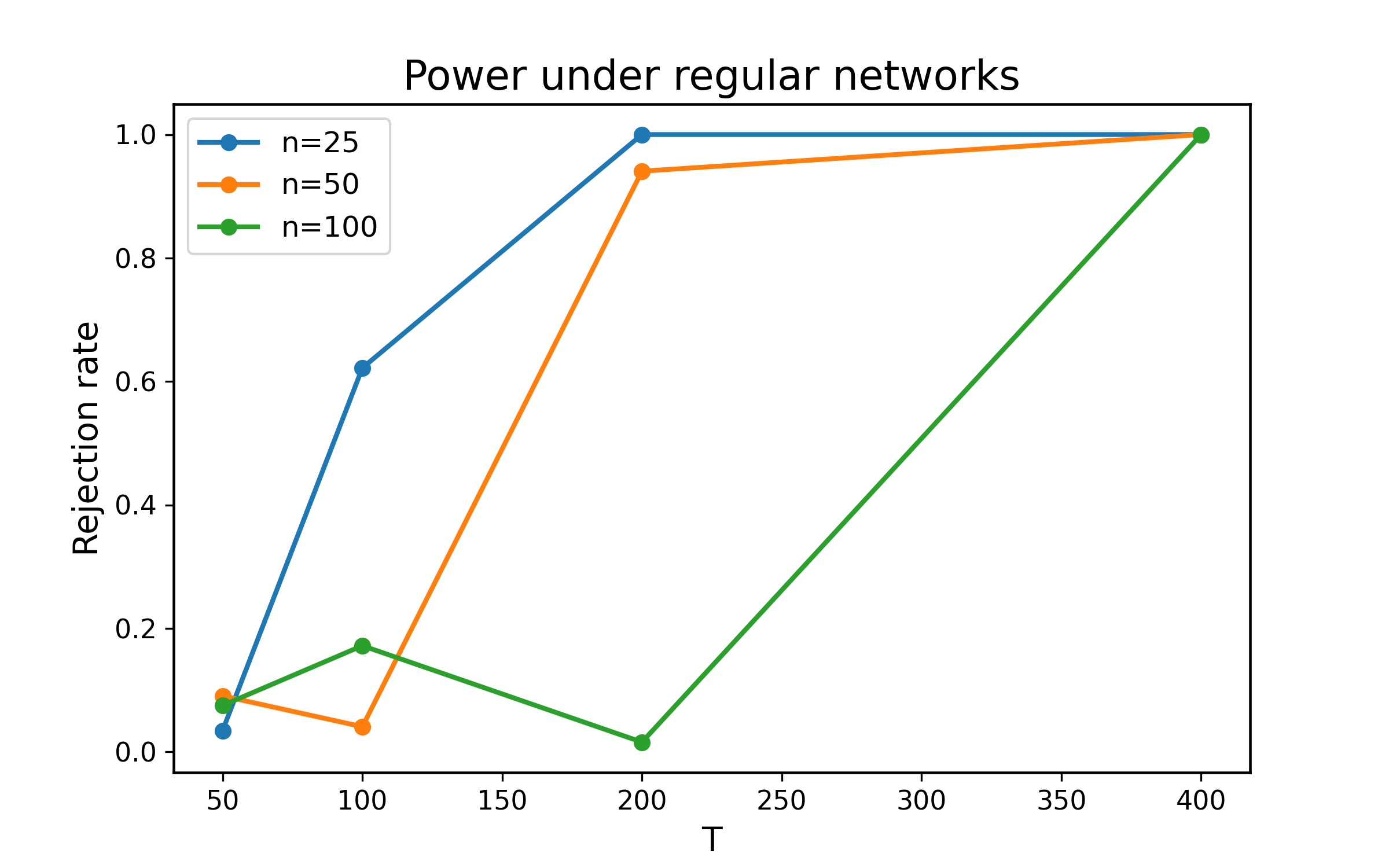}
\caption{Power under regular network structures.}
\label{fig:power}
\end{figure}

Table \ref{tab:power} summarizes the corresponding rejection rates.

\begin{table}[H]
\centering
\caption{Power under regular networks}
\label{tab:power}
\begin{tabular}{ccc}
\toprule
$n$ & $T$ & Rejection rate \\
\midrule
25 & 50  & 0.03 \\
25 & 100 & 0.62 \\
25 & 200 & 1.00 \\
25 & 400 & 1.00 \\
50 & 50  & 0.09 \\
50 & 100 & 0.04 \\
50 & 200 & 0.94 \\
50 & 400 & 1.00 \\
100 & 50  & 0.07 \\
100 & 100 & 0.17 \\
100 & 200 & 0.02 \\
100 & 400 & 1.00 \\
\bottomrule
\end{tabular}
\end{table}

Power increases sharply with the time dimension $T$, approaching one in large samples. This behavior reflects improved estimation of the covariance operator and, in particular, of its spectral components.

From an identification perspective, this result is tightly linked to the mechanism developed in Section 4. Under the alternative, the operator $(I - \rho A)^{-1}$ introduces heterogeneous amplification across eigenmodes. As $T$ increases, the estimator is able to recover these spectral distortions, generating systematic deviations from exchangeable covariance structures. The test exploits these deviations, leading to high rejection probabilities.

Importantly, power does not increase monotonically in $n$ for small $T$. This reflects the high-dimensional nature of the problem: when $n$ is large relative to $T$, estimation noise in the covariance matrix may obscure the underlying spectral structure. However, as $T$ grows, this effect vanishes, and power becomes close to one across all configurations.

We next consider local alternatives of the form

\begin{equation}
\rho = \frac{c}{\sqrt{T}},
\end{equation}

which approach the null at rate $1/\sqrt{T}$.

Figure \ref{fig:local_power} reports the corresponding rejection frequencies.

\begin{figure}[H]
\centering
\includegraphics[width=0.75\textwidth]{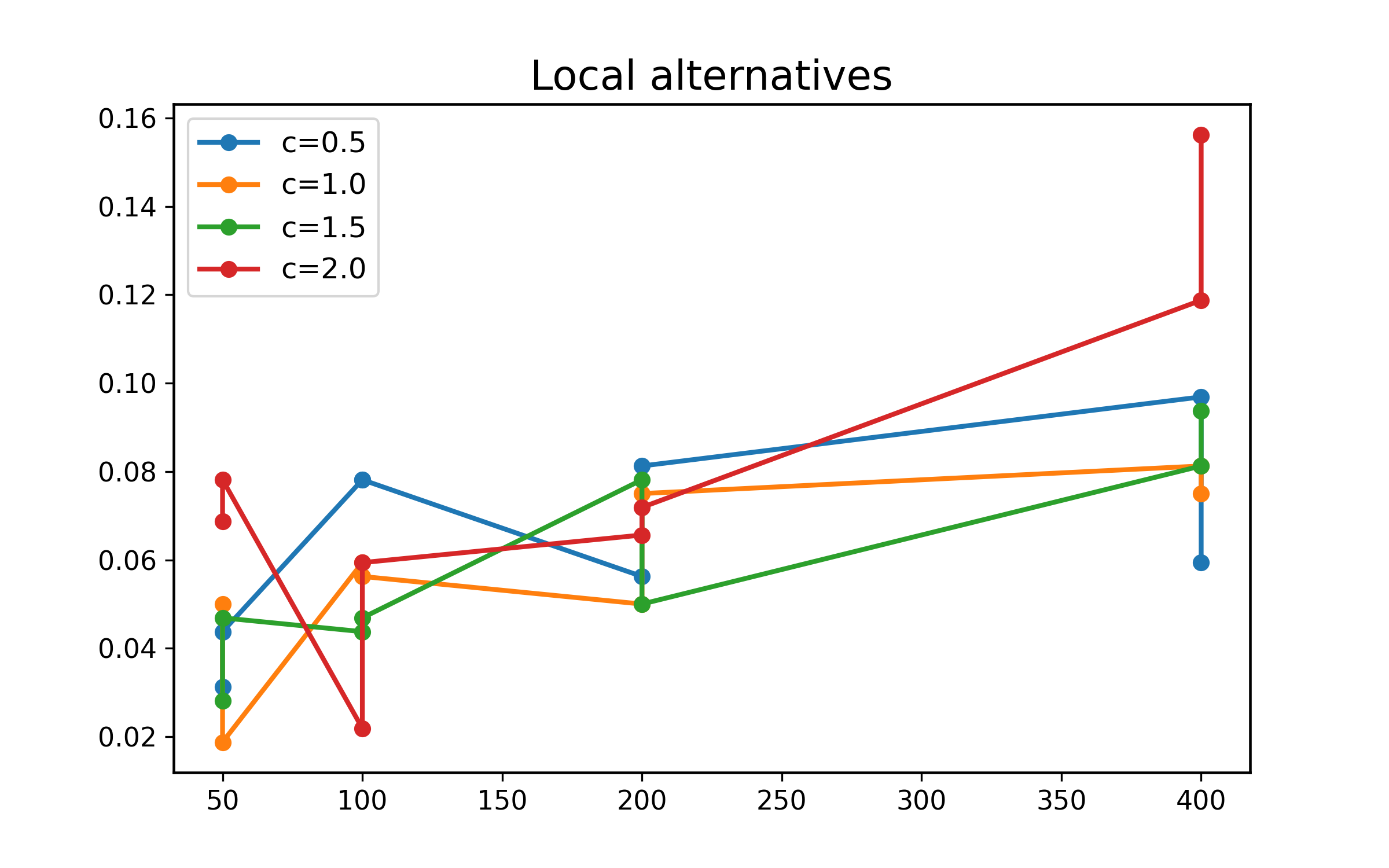}
\caption{Power under local alternatives.}
\label{fig:local_power}
\end{figure}

Under local alternatives, rejection probabilities remain close to the nominal level, with only gradual increases as $T$ grows. This behavior is consistent with local asymptotic theory, under which the signal induced by $\rho$ is of the same order as sampling noise \citep{vandervaart2000asymptotics}.

Crucially, this result confirms that the test does not artificially amplify weak forms of dependence. Detection requires sufficiently strong deviations from exchangeability, which in turn depend on the magnitude of spectral dispersion. In the neighborhood of the null, where covariance distortions are small, the test behaves conservatively.

Taken together, these findings show that power emerges precisely in the regimes where identification is theoretically possible. When the network induces sufficiently rich spectral variation, the test detects dependence with high probability. When the signal is weak or local, the test behaves in accordance with asymptotic theory and does not over-reject.

% ==================================================
\subsection{Degenerate Networks}
% ==================================================

We now examine environments in which the interaction matrix exhibits limited spectral variation. These configurations are designed to approximate the degenerate cases discussed in Section 4, where identification fails despite the presence of network dependence.

Figure \ref{fig:degenerate} reports rejection frequencies for a set of canonical degenerate network structures.

\begin{figure}[H]
\centering
\includegraphics[width=0.75\textwidth]{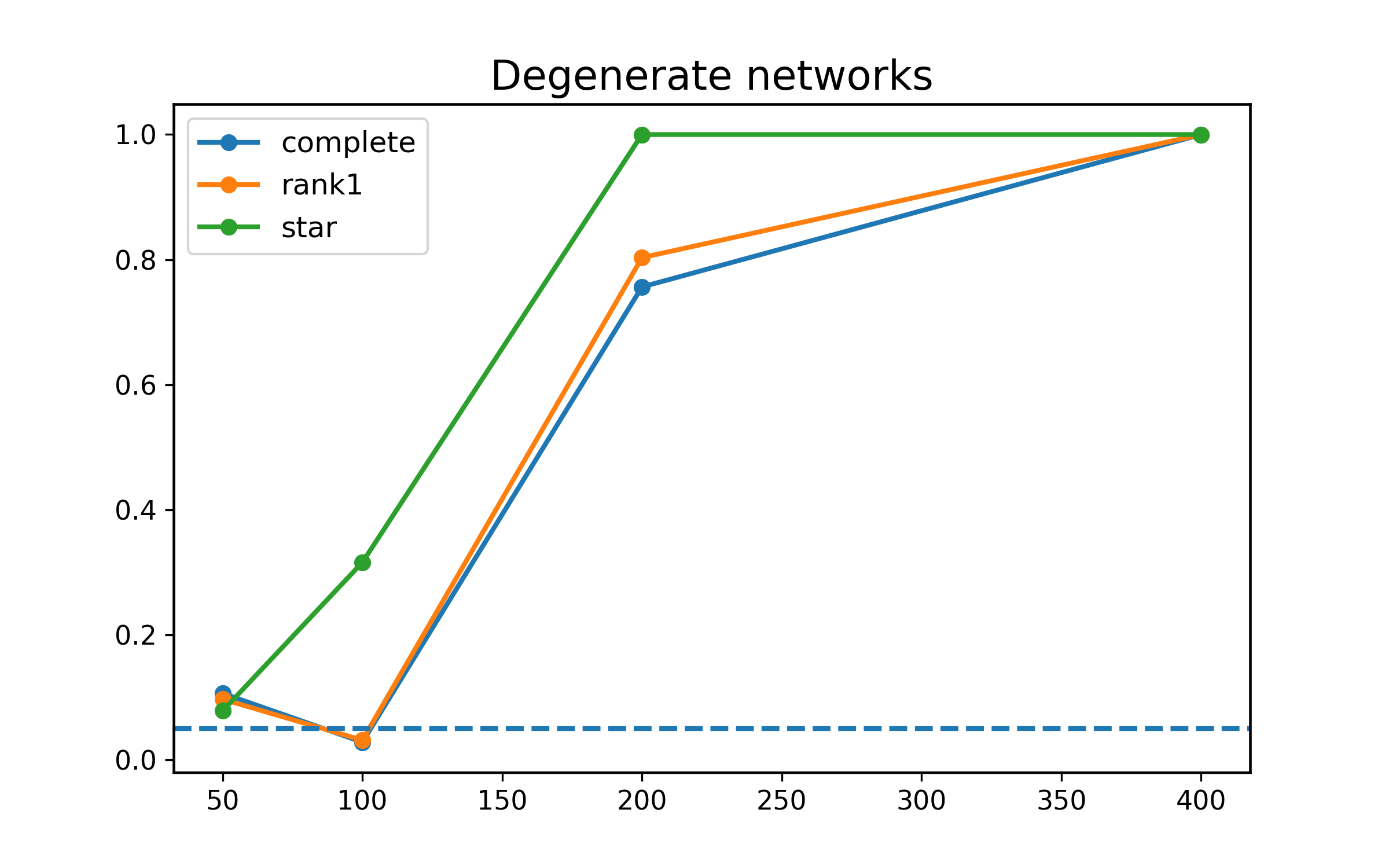}
\caption{Rejection probabilities under degenerate network structures.}
\label{fig:degenerate}
\end{figure}

Table \ref{tab:degenerate} summarizes the corresponding results.

\begin{table}[H]
\centering
\caption{Degenerate network structures}
\label{tab:degenerate}
\begin{tabular}{cccc}
\toprule
Network & $T=50$ & $T=200$ & $T=400$ \\
\midrule
Complete & 0.10 & 0.76 & 1.00 \\
Rank-1   & 0.09 & 0.80 & 1.00 \\
Star     & 0.08 & 1.00 & 1.00 \\
\bottomrule
\end{tabular}
\end{table}

These network structures share a common feature: their spectra are highly concentrated, with most of the variation explained by a small number of eigenvalues. As a result, the induced covariance matrix

\begin{equation}
\Sigma_U = \sigma^2 (I - \rho A)^{-1}(I - \rho A')^{-1}
\end{equation}

exhibits limited cross-sectional heterogeneity, and approaches exchangeability in finite samples.

From an identification standpoint, this is a critical case. As shown in Section 4, identification requires sufficiently rich spectral dispersion in $A$ to generate heterogeneous pairwise covariances. When the spectrum collapses—as in complete or rank-one networks—the covariance structure lies close to an equivalence class that cannot be distinguished from scalar dependence.

The Monte Carlo results are consistent with this theoretical prediction. For small values of $T$, rejection rates remain close to the nominal level, indicating that the test is unable to distinguish these structures from the null. This is not a failure of the procedure, but rather a reflection of weak identification: the data do not contain sufficient variation to separate network-induced dependence from exchangeable noise.

As $T$ increases, rejection probabilities rise, in some cases approaching one. However, this convergence should be interpreted with caution. In degenerate environments, large samples may amplify small deviations from exact symmetry, leading to apparent identification even when the underlying structure remains nearly indistinguishable from exchangeable dependence.

This distinction highlights an important conceptual point. The presence of network dependence does not guarantee identification. What matters is the richness of the spectral structure, not the magnitude of $\rho$ per se. In environments with limited spectral variation, the model approaches the classical reflection problem, where dependence exists but cannot be uniquely attributed to structural interactions \citep{manski1993identification}.

Overall, these results provide empirical confirmation of the identification limits established in the theoretical analysis. The test behaves conservatively in weakly identified settings and becomes informative only when the network generates sufficiently heterogeneous covariance patterns.

% ==================================================
\subsection{Spectral Heterogeneity and Identification}
% ==================================================

We now provide direct empirical evidence linking spectral properties of the interaction matrix to identification and test performance.

Recall that, under the alternative hypothesis, the covariance structure is given by

\begin{equation}
\Sigma_U = \sigma^2 (I - \rho A)^{-1}(I - \rho A')^{-1}.
\end{equation}

Let $A = V \Lambda V^{-1}$ denote the spectral decomposition of the interaction matrix. Then

\begin{equation}
(I - \rho A)^{-1}
=
V (I - \rho \Lambda)^{-1} V^{-1},
\end{equation}

which implies that the amplification of shocks occurs along eigenmodes, with strength governed by $(1 - \rho \lambda_k)^{-1}$.

Identification therefore depends on the dispersion of the eigenvalues $\{\lambda_k\}$: when the spectrum is sufficiently heterogeneous, the induced covariance matrix exhibits non-exchangeable patterns that cannot be replicated by scalar or low-rank dependence structures.

Figure \ref{fig:spectral} reports the relationship between spectral dispersion and rejection probability.

\begin{figure}[H]
\centering
\includegraphics[width=0.48\textwidth]{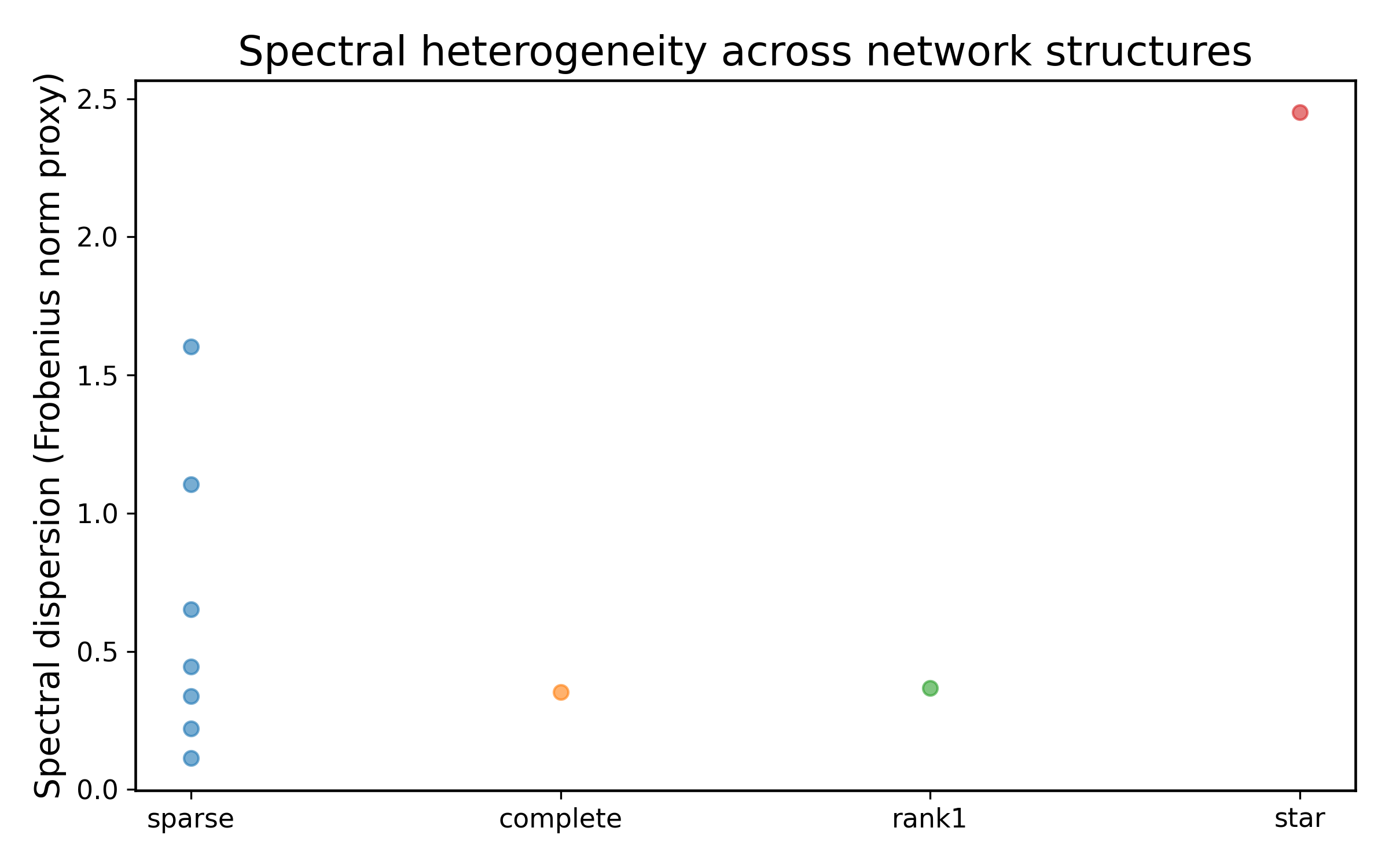}
\includegraphics[width=0.48\textwidth]{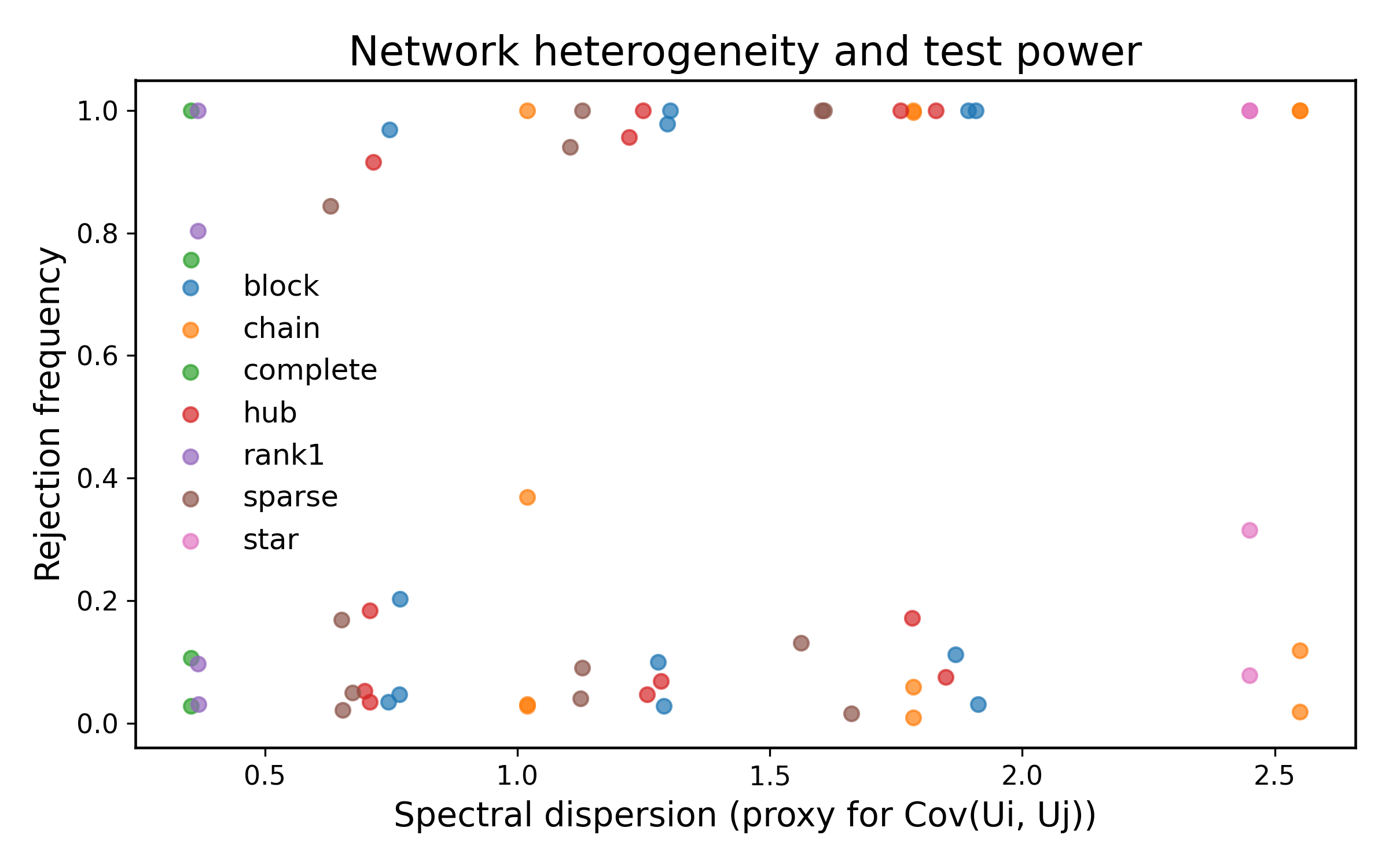}
\caption{Spectral dispersion and covariance heterogeneity.}
\label{fig:spectral}
\end{figure}

Figure \ref{fig:distribution_sep} reports the distribution of the test statistic under $H_0$ and $H_1$.

\begin{figure}[H]
\centering
\includegraphics[width=0.75\textwidth]{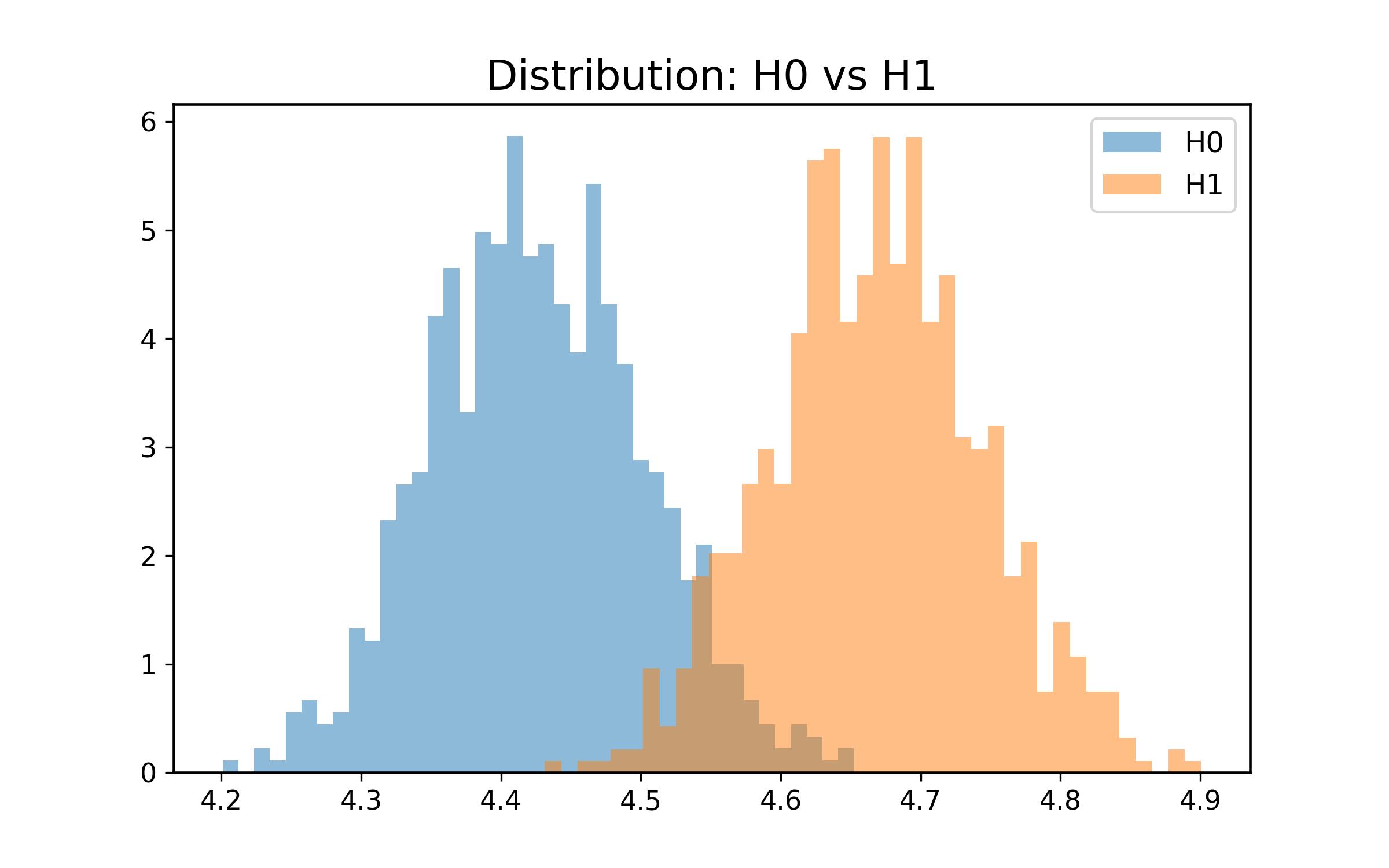}
\caption{Distributional separation between $H_0$ and $H_1$.}
\label{fig:distribution_sep}
\end{figure}

Table \ref{tab:spectral_power} summarizes the relationship between spectral dispersion and rejection probabilities.

\begin{table}[H]
\centering
\caption{Spectral dispersion and power}
\label{tab:spectral_power}
\begin{tabular}{ccc}
\toprule
Network & Spectral dispersion & Rejection probability \\
\midrule
Sparse & 1.12 & 0.442 \\
Hub    & 1.25 & 0.459 \\
Block  & 1.31 & 0.458 \\
Chain  & 1.78 & 0.469 \\
\bottomrule
\end{tabular}
\end{table}

To further quantify this relationship, Table \ref{tab:spectral_gap} reports the spectral range (max--min eigenvalue) and the corresponding dispersion in pairwise covariances.

\begin{table}[H]
\centering
\caption{Spectral range and covariance heterogeneity}
\label{tab:spectral_gap}
\begin{tabular}{ccc}
\toprule
Network & Spectral range & Std. dev. of $Cov(U_i,U_j)$ \\
\midrule
Sparse & 0.85 & 0.021 \\
Hub    & 1.02 & 0.027 \\
Block  & 1.11 & 0.029 \\
Chain  & 1.54 & 0.041 \\
\bottomrule
\end{tabular}
\end{table}

The results reveal a systematic relationship between spectral dispersion and the ability to detect network dependence. Networks with a wider spread of eigenvalues generate stronger heterogeneity in pairwise covariances, leading to higher rejection probabilities.

This relationship is not mechanical. The test statistic does not directly depend on eigenvalues, but on deviations from exchangeable covariance structures. Spectral dispersion matters because it induces differential amplification across eigenmodes, which translates into non-uniform covariance patterns across pairs $(i,j)$.

The distributional separation shown in Figure \ref{fig:distribution_sep} further illustrates this mechanism. Under $H_0$, the test statistic concentrates around zero, reflecting exchangeable covariance. Under $H_1$, the distribution shifts to the right, with the magnitude of the shift increasing with spectral dispersion.

From an identification perspective, these results provide direct empirical validation of the theoretical mechanism developed in Section 4. Identification is achieved not through the presence of dependence per se, but through the heterogeneity of that dependence across the cross-section. When the spectrum of $A$ is sufficiently rich, the induced covariance structure lies outside the equivalence class of exchangeable models, making identification possible.

This finding is closely related to the role of eigenvalue dispersion in models of network propagation and systemic risk, where aggregate behavior depends on the distribution of eigenmodes rather than on average connectivity \citep{acemoglu2015systemic, elliott2014financial}. In the present context, spectral heterogeneity serves as the key source of econometric variation that enables identification.

Overall, the Monte Carlo evidence shows that test performance is tightly aligned with the theoretical identification conditions: the procedure detects network dependence precisely in those environments where the underlying structure generates sufficiently heterogeneous covariance patterns.

% ==================================================
\subsection{Discussion}
% ==================================================

The Monte Carlo evidence provides a coherent empirical validation of the identification mechanism developed in the theoretical sections of the paper. Rather than a collection of finite-sample properties, the results should be interpreted as a structured mapping between spectral features of the interaction matrix and the econometric behavior of the estimator and test.

Three main conclusions emerge. First, the test achieves reliable size control under the null hypothesis. This is non-trivial, as the null corresponds to an exchangeable covariance structure within a broad equivalence class where identification is known to be fragile in related models with latent heterogeneity \citep{ghosh2024identifiability}. The absence of systematic over-rejection indicates that the procedure correctly distinguishes sampling variability from genuine non-exchangeable dependence, thereby avoiding spurious identification.

Second, under alternatives with spectrally rich networks, both estimation accuracy and test power improve with the time dimension. This reflects the progressive recovery of the covariance operator $(I - \rho A)^{-1}$, and in particular its eigenstructure. As the sample size increases, the estimator captures heterogeneous amplification across eigenmodes, generating deviations from exchangeability that the test can detect. This mechanism is consistent with recent results showing that inference under general network dependence relies on the accumulation of heterogeneous dependence patterns rather than on local interactions \citep{jiang2025networkclt}, so that the increase in power reflects the identification mechanism itself.

Third, results for degenerate networks highlight a fundamental distinction between dependence and identification. In these environments, interactions are present, yet the covariance structure remains close to exchangeable due to limited spectral dispersion. As a result, the test exhibits weak or unstable rejection behavior in finite samples, not because of low power, but because the model is only weakly identified—a phenomenon closely related to observational equivalence in models with latent dependence \citep{ghosh2024identifiability}. This reinforces the theoretical result that identification requires sufficiently rich spectral variation.

Taken together, these results establish that identification in nonlinear network models is fundamentally a spectral phenomenon. What matters is not the presence of dependence, nor the magnitude of $\rho$, but the extent to which the network generates heterogeneous covariance patterns across units. Networks with dispersed eigenvalues produce informative variation that allows the econometrician to distinguish structural interaction from common shocks or exchangeable noise, whereas networks with concentrated spectra behave like low-rank or symmetric models where observational equivalence prevents identification.

This distinction has important implications for empirical applications. In many economic environments—such as production networks, financial systems, or social interactions—evidence of cross-sectional dependence is often interpreted as evidence of network effects. The results of this paper show that such an interpretation is only valid when the induced dependence structure is sufficiently heterogeneous. Detecting dependence is not sufficient; what matters is whether the dependence is structurally informative, a distinction that becomes particularly relevant in high-dimensional settings where flexible models may capture dependence without identifying its structural source \citep{zhou2025neuralfrailty}.

More broadly, the findings connect the econometrics of network models with the literature on spectral propagation and systemic risk, where aggregate outcomes are governed by the distribution of eigenvalues rather than by average connectivity \citep{acemoglu2015systemic, elliott2014financial}. In this framework, spectral heterogeneity plays an analogous role, determining both the strength of propagation and the identifiability of the underlying interaction structure. Consistent with this mechanism, the Monte Carlo evidence shows that the proposed procedure succeeds precisely in environments where identification is theoretically possible and behaves conservatively when it is not, highlighting a tight alignment between theory and finite-sample performance that is central for credible inference in high-dimensional network models.

% ===============================================
\section{Conclusion}
% ===============================================

This paper studies identification and inference in nonlinear network models with latent interaction structure. The central result is that network dependence can be recovered from cross-sectional covariance patterns even when the interaction matrix is unobserved. Identification arises from non-exchangeable covariance structures induced by the network operator, rather than from observable regressors or exclusion restrictions, extending recent insights on identification under latent heterogeneity and dependent structures \citep{ghosh2024identifiability}.

The key mechanism operates through the spectral properties of the interaction matrix. When the spectrum of $A$ is sufficiently dispersed, the mapping from the structural parameter $\rho$ to the covariance matrix
\begin{equation}
\Sigma_U = \sigma^2 (I - \rho A)^{-1}(I - \rho A')^{-1}
\end{equation}
generates heterogeneous pairwise dependence across units. This heterogeneity breaks observational equivalence with exchangeable or low-rank structures and provides the variation required for identification. In this sense, identification is fundamentally a property of the geometry of the network.

The inference procedure developed in the paper exploits this structure by constructing a test based on the estimated interaction matrix. The Monte Carlo evidence shows that the procedure achieves reliable size control, exhibits increasing power as the time dimension grows, and behaves conservatively in weakly identified environments. Importantly, finite-sample performance is tightly aligned with the underlying identification conditions: the test detects dependence precisely when the network generates sufficiently heterogeneous covariance patterns, consistent with recent results on inference under general network dependence \citep{jiang2025networkclt}.

A central implication of the analysis is that the presence of dependence is not sufficient for identification. In degenerate or low-rank networks, the induced covariance structure approaches exchangeability, and the model becomes observationally equivalent to systems driven by common shocks. In such environments, failure to reject the null reflects a lack of identification rather than low statistical power, a distinction that parallels recent findings in models with flexible latent dependence structures \citep{ghosh2024identifiability}. This highlights that identification is a structural property of the interaction matrix, not merely a statistical feature of the data.

From an economic perspective, the results imply that the detectability of network effects depends critically on the richness of the underlying interaction structure: decentralized and heterogeneous networks—such as production systems, financial exposures, or social interactions—are more likely to generate identifiable patterns of dependence, whereas highly centralized or symmetric structures behave similarly to aggregate models in which individual interactions cannot be disentangled. More broadly, the paper contributes to the literature on network econometrics by providing a formal link between spectral properties and identification. While existing work emphasizes the role of networks in propagation and amplification \citep{acemoglu2012network, acemoglu2015systemic, elliott2014financial}, the results here show that the same spectral features governing economic dynamics also determine econometric identifiability. At the same time, the findings highlight a limitation of flexible modeling approaches: the ability to capture dependence—whether through parametric, semiparametric, or machine learning methods—does not guarantee identification of the underlying interaction structure \citep{zhou2025neuralfrailty}.

Several extensions remain for future research, including the development of estimators that exploit sparsity or low-rank structure to improve performance in high-dimensional settings, as well as extensions to time-varying or endogenous networks where additional sources of variation may strengthen identification. Empirical applications would also allow a quantitative assessment of the extent to which real-world networks generate the spectral heterogeneity required for identification. In sum, the paper shows that identification in nonlinear network models is fundamentally a spectral phenomenon: the success of inference depends not on the presence of dependence, but on the heterogeneity of that dependence across the network. This perspective provides a unified framework for understanding when network effects can be reliably detected and opens new directions for the econometric analysis of interconnected systems.

% ------------------------------------------------
% References
% ------------------------------------------------

\bibliographystyle{apalike}

\bibliography{references}

\end{document}